\documentclass{article} % For LaTeX2e
\usepackage{iclr2016_conference,times}
\usepackage{hyperref,url}
\usepackage{amsmath,amsbsy,amsfonts,amssymb,amsthm,mathtools,nicefrac}
\usepackage{algorithm,algorithmic}
\usepackage{graphicx,subfigure,tikz}
\usepackage{geometry}

\newgeometry{bottom= 1.6in, top=1.25in}
\renewcommand{\cite}{\citep}

\newtheorem{theorem}{Theorem}[section]
\newtheorem{claim}[theorem]{Claim}

\newtheorem{lemma}[theorem]{Lemma}
\newtheorem{proposition}[theorem]{Proposition}

\theoremstyle{definition}

\newtheorem{definition}{Definition}

\def\shownotes{0}  %set 1 to show author notes
\ifnum\shownotes=1
\newcommand{\authnote}[2]{{ $\ll$\textsf{\footnotesize #1 notes: #2}$\gg$}}
\else
\newcommand{\authnote}[2]{}
\fi

\newcommand\E{\mathbb{E}}
\newcommand\R{\mathbb{R}}

\newcommand{\Exp}{\mathop{\mathbb E}\displaylimits}

\newcommand{\psinorm}[2]{\|#1\|_{\psi_{#2}}}

\newcommand{\tilO}{\widetilde{O}}

\newcommand{\relu}{r}

\newcommand{\one}{\mathbf{1}}
\newcommand{\Var}{\textup{Var}}

\newcommand{\drop}{\textrm{drop}}

\title{Why are deep nets reversible: A simple theory, with implications for training}

\author{Sanjeev Arora, Yingyu Liang \& Tengyu Ma\\
Department of Computer Science\\
Princeton University\\
Princeton, NJ 08540, USA \\
\texttt{\{arora,yingyul,tengyu\}@cs.princeton.edu} 
}

\begin{document}

\maketitle

% !TEX root = deepnet_main.tex
\vspace*{-0.1in}
\begin{abstract}
 Generative models for  deep learning are promising  both to
  improve understanding of the model, and yield training methods requiring fewer labeled samples.
 % and because they may cast new light on deep learning. 
 %\Tnote{I guess we could also say ``label''-efficient. }
 
 Recent works use generative model approaches to produce the deep net's input given the value of a hidden layer several levels above. 
 % \Tnote{The word ``reconstruct'' might cause some confusion. maybe replace it by ``produce" or ``generate"? } 
 However, there is no accompanying \textquotedblleft proof of correctness\textquotedblright\ for the generative model, showing
 that  the feedforward deep net is the correct inference method for recovering the hidden layer given the input.  Furthermore, these models are complicated.

 The current paper takes a more {\em theoretical} tack. It presents a very simple generative model
 for RELU deep nets, with the following characteristics: 
 (i) The generative model is just the {\em reverse} of the feedforward net: if the forward transformation at a layer is $A$ then the reverse transformation is $A^T$. (This can be seen as an explanation of the old 
 {\em weight tying} idea for denoising autoencoders.)
  (ii) Its correctness can be {\em proven}
under a clean theoretical assumption: the edge weights in real-life deep nets behave like random numbers. Under this assumption ---which is experimentally tested on real-life nets like AlexNet--- it is formally proved that feed forward net is a correct inference method for  recovering the hidden layer. 

 The generative model suggests a simple modification for training: use the generative model to produce  synthetic data with labels and  include it in the training set. 
 Experiments are shown to support this theory of random-like deep nets; and that it helps the training.

 \iffalse 
 The theory also provides new insight into {\em dropout} since robustness to dropout is a property of random-like nets.
 It also suggests a new regularizer to use in addition to dropout. Empirical results are shown suggesting that 
 the combination of this regularizer with dropout provides better results than dropout, especially in settings with not enough
 data (such as CIFAR100).
 \fi  
\end{abstract}

\vspace{-2mm}
\section{Introduction}
\label{sec:intro}

Discriminative/generative pairs of models for classification tasks are an old theme in machine learning~\cite{JordanNg01}.
\iffalse . Discriminative models make no assumptions about the structure of the input, and often yield better results given enough labeled data, whereas generative models can help when labeled data is scarce. 
This paradigm has motivated  much
\fi  Generative model analogs for deep learning may not only cast new light on the discriminative backpropagation algorithm,   but also allow learning
with fewer labeled samples. 
%  especially since discriminative training seems to require 
%large number of labeled examples.
 A seeming obstacle in this quest is that deep nets are successful in a variety of domains, 
and it is unlikely that problem inputs in these domains share common families of generative models. 

Some generic (i.e., not tied to specific domain) approaches to defining such models include {\em Restricted Boltzmann Machines}~\cite{FreundHaussler94,hinton2006reducing} and {\em Denoising Autoencoders}~\cite{BengioLamblin,DBLP:conf/icml/VincentLBM08}. Surprisingly, these suggest that
deep nets are {\em reversible}: the generative model is a essentially the feedforward net run in reverse. %There does not seem to be an explanation for why such reversible nets fit problems in varied domains. 
% \Tnote{Reviewers might ask that RBM is reversible by definition and it's indeed provable}
Further refinements include Stacked Denoising Autoencoders~\cite{vincent2010stacked}, Generalized Denoising Auto-Encoders~\cite{Bengio13} and Deep Generative Stochastic Networks~\cite{bengio2013deep}. %{\sc cite more things here}. %\yingyu{cannot find the citation baraniuketal15}\Tnote{added baraniuk's paper though are we citing it here?}
\iffalse Fitting such a model to the data can be a possible approach to unsupervised learning, or at least do unsupervised pretraining to give a starting point net to train with supervised methods. Despite notable successes, finding a good generative model in this sense ---one that comes with a fast training algorithm---  is still generally considered an open problem.
\fi 

%Autoencoders, Vector Quantization, Restricted Boltzman Machines, Denoising Autoencoders, etc. have been suggested as 

In case of image recognition it is possible to work harder ---using a custom deep net to invert the feedforward net ----and {\em reproduce}  the input very well from the values of hidden layers much higher up, and in fact to generate images very different from any that were used to train the net (e.g.,~\cite{Mahendran15}).  

\iffalse 
Another setback for the unsupervised approach has been the empirical finding
that simple back-propagation with random initialization (given enough labeled data) is tough to beat.
Even more so if one adds dropout~\cite{Srivastavadropout}: randomly remove half the nodes from the net and expect the remaining net to
perform just as well. (See~\cite{wager2013dropout} on viewing dropout as a regularizer, albeit not in a neural net setting.) 
\fi 

To explain the contribution of this paper and contrast with past work,  we need to formally define the problem. Let $x$ denote the data/input to the deep net and  $z$ denote the hidden representation (or the output labels). 
%\Tnote{Changed the definitio a bit. input/output might be confusing}
%(If the output happens to be binary, think of $z$ as the vector of values at the last-but-one layer.)  
The generative model has to satisfy the following: {\bf Property (a)}: Specify a joint distribution 
of $x, z$, or at least  $p(x|z)$. {\bf Property (b)} A proof that the deep net itself is a method of computing the 
(most likely) $z$ given $x$. As explained in Section~\ref{sec:relatedwork}, past work usually fails to satisfy one of (a) and (b).

\iffalse 
Such a model implies an obvious method to generate synthetic data given the net: Take an input $x$, use the feedforward net to produce a label $z$, and then sample using the model $p(x|z)$ to generate a synthetic input $\tilde{x}$. {\em By construction} the  net must assign the  label $z$ to $\tilde{x}$ as well.
%this as $z$. Now $\tilde{x}, z$ is a labeled datapoint that could potentially be used in supervised training. 
\fi

%\Tnote{~\cite{baldiunderstandingdropout} also tries to explain dropout and they got oral at last nips. I don't like the paper much but we may cite it. I can't make any sense of this paper though }
The current paper introduces a simple mathematical 
explanation for {\em why} such a model should exist for deep nets with fully connected layers. 
%deep nets should have a generative analog, and in fact
%one that is the {\em reverse} of the feedforward net. 
We propose the {\em random-like nets hypothesis}, which says  that real-life deep nets --even those obtained from standard supervised learning---are \textquotedblleft random-like,\textquotedblright\ meaning their edge weights behave like random numbers. Notice, this is distinct from saying that the edge weights actually {\em are} randomly generated or uncorrelated. Instead we mean that the weighted graph 
has bulk properties similar to those of random weighted graphs. To give an example, matrices in a host of settings are known to display properties ---specifically, eigenvalue distribution--- similar to matrices with
Gaussian entries; this so-called {\em Universality} phenomenon is a matrix analog of the Law of Large Numbers. The random-like properties of deep nets needed in this paper (see proof of Theorem~\ref{thm:single}) involve a generalized eigenvalue-like property.

% real-life deep nets (e.g., see Figure~\ref{}) as well.

If a deep net is  random-like, we can show mathematically (Section~\ref{sec:multilayer})that it has an associated simple generative model $p(x|z)$ (Property (a)) that we call the {\em shadow distribution}, and for which Property (b) also {\em automatically} holds in an approximate sense. (Our proof currently works for up to $3$ layers, though experiments show Property (b) holds for even more layers.)
%This generative model is very reminiscent of denoising autoencoders, except there Property (b)
%was not proved. 
Our generative model makes essential use of dropout noise and RELUs  and can be seen as providing (yet another) theoretical explanation for the efficacy of these two in modern deep nets.

Note that Properties (a) and (b) hold even for random (and hence untrained/useless) deep nets. Empirically, supervised training seems to improve the shadow distribution, and at the end the synthetic images are somewhat reasonable, albeit  cruder compared to say~\cite{Mahendran15}.

The knowledge that the deep net being sought is random-like can be used to improve training.
Namely, take a labeled data point $x$, and use the current feedforward net to compute its label $z$. Now use the shadow distribution $p(x|z)$ to compute a {\em synthetic}
data point $\tilde{x}$, label it with $z$, and add it to the training set for the next iteration. 
We call this the SHADOW method. 
Experiments reported later show that adding this to training yields 
measurable improvements over backpropagation + dropout for training fully connected layers. Furthermore, throughout training,  the prediction error on synthetic data closely tracks that on the real data, as predicted by the theory.

\vspace{-2mm}
\subsection{Related Work and Notation}
\label{sec:relatedwork}

%Let $x$ denote the input to a deep net and $z$ the output, 
%The hope in defining such a generative model would be to use bayesian techniques  to find the best net. 
%are related to Olshausen and Field's 
%{\em sparse coding} framework, which can be seen as a generative analog of a one-layer net. 
Deep Boltzmann Machine(~\cite{HintonOT06}) is  an attempt to define a generative model in the above 
sense, and is related to the older notion of {\em autoencoder}.  For a single layer, it posits a joint distribution of the 
observed layer $x$ and the hidden layer $h$ of the form $\exp(-h^TAx)$. 
This makes it {\em reversible}, in the sense that conditional distributions $x~|~h$ and
$h~|x$ are both easy to sample from, using essentially the same distributional form and with shared parameters.
Thus a single layer net indeed satisfies properties (a) and (b).
 To get a multilayer net, the DBN stacks such units. But as pointed out in~\cite{bengiosurvey}
this does not yield a generative model per se since one cannot ensure that the marginal distributions of two adjacent layers match. If one changes  the generative model to match up the conditional probabilities, then one loses reversibility. % of RBM.
%In other words, if $h$ is some layer in between $x$ and $z$, then the
%forward probability $p_{forward}(h|x)$ in general does not match the backward probability
%$p_{backward}(h|z)$. 
%\Tnote{Edited sentence above. It was a bit confusing to me -- there is a way to make the generative model well-defined by only using the conditional probablity (which is exactly what DBN did). Though it's not reversible}
Thus the model violates one of (a) or (b). We know of no prior solution to this issue.
% (In the current paper such issues can  resolved  under the random-like nets hypothesis.)

In~(\cite{LeeNgetal09}) the RBM notion is extended to convolutional RBMs. 
In~(\cite{NairHinton12}), the theory of RBMs is extended to allow rectifier linear units, but this involves
approximating RELU's with multiple binary units, which seems inefficient. (Our generative model below will sidestep this inefficient conversion to binary and work directly with RELUs in forward and backward direction.) 
Recently, a sequence of papers define hierarchichal probabilistic models~(\cite{RanganathTCB15,KingmaMRW14,Baraniuk15}) that are plausibly reminiscent of standard deep nets.
Such models can be used to model very lifelike distributions on documents or images, and thus some of them plausibly satisfy Property (a).
But they are not accompanied by any proof that the feedforward deep net solves the inference problem of recovering the top layer
$z$, and thus don't satisfy (b).

The paper of~(\cite{ABGM14}) defines a consistent generative model satisfying (a) and (b) under some restrictive conditions: the neural net edge weights are {\em random} numbers, and the connections satisfy some conditions on degrees, sparsity of each layer etc.  However, the restrictive conditions limit its applicability to real-life nets. The current paper doesn't impose such restrictive conditions.
%, but then loses provably efficient training ----we use backpropagation, which has no proof of
%fast convergence.

\iffalse  but it is unclear how to relate these
models to real life nets, and furthermore, how to use Bayesian techniques to compete with the performance of even plain back-propagation on large data sets. \fi 
% large data sets
%that current deep net techniques can handle.  

Finally, several works have tried to show that the deep nets can be inverted, such that the observable layer can be recovered from its representation 
at some very high hidden layer of the net~(\cite{Lee09}).
The recovery problem is solved in the recent paper~(\cite{Mahendran15}) also using a deep net. %\Tnote{Didn't understand this sentence. Will get back to this}
While very interesting, this does not define a generative model (i.e., doesn't satisfy Property (a)).
 {\em Adversarial nets}~\cite{Goodfellow14} can be used to define very good generative models for natural images, but don't attempt to satisfy Property (b). (The discriminative net there doesn't do the inference but just attempts to distinguish between real data and generated one.)
%either of Properties (a) and (b).

{\bf Notation: } We use $\|\cdot\|$ to denote the Euclidean norm of a vector and spectral norm of a matrix. Also, $\|x\|_{\infty}$ denotes infinity (max) norm of a vector $x$, and $\|x\|_0$ denotes the number of non-zeros.  %$\|\cdot\|_q$  to denote the $q$-norm of a vector, and $|\cdot|_0$ is the number of nonzero entries of a vector.
 The set of integers $\{1,\dots,p\}$ is denoted by $[p]$. %We write $M\succeq 0$ if $M$ is a positive semidefinite matrix. $\polyringp{d}{n}$ is used to denote the set of real polynomials with $n$ variables and degree at most $d$. We will drop the subscript $n$ when it is clear from context. 
 Our calculations use asymptotic notation $O(\cdot)$, which assumes that parameters such as the size of network, the total number of nodes (denoted by $N$), the sparsity parameters (e.g $k_j$'s introduced later) are all sufficiently large. (But $O(\cdot)$
 notation will not hide any impractically large constants.)
 % \Tnote{$O(\cdot)$ only hides absolute constant?}
 %\footnote{Or we assume that they go to infinity as typically do in statistics. }, and $n\le p$. We use the notation $\tilO(\cdot)$ and $\tilOmega(\cdot)$ to hide the logarithmic dependency (in $p$). That is, $m \le \tilO(f(n,p,k))$ means that there exists universal constant $r\ge 0$ (which is less than $3$ typically in this paper) and $C$ such that $m \le Cf(n,p,k)\log^r p$, and $m \ge \tilOmega(f(n,p,k))$ means that there exist constants $r$ and $c$ such that $m\ge cf(n,p,k)/\log^r p$. %\Anote{Fixed the Omega explanation- I assume this is what you meant} 
Throughout, the \textquotedblleft with high probability  event $E$ happens\textquotedblright\ means the failure probability is upperbounded by $N^{-c}$ for constant $c$, as $N$ tends to infinity. Finally, $\tilO(\cdot)$ notation hides terms that depend on $\log N$. %which is not important for the purpose of this paper. 

% !TEX root = deepnet_main.tex
\vspace*{-0.1in}
\section{Single layer generative model}
\label{sec:model}

For simplicity let's
start with a single layer neural net  $h = r(W^Tx + b),$ where $r$ is the rectifier linear function, $x \in \R^n, h \in \R^m$. When $h$ has fewer nonzero coordinates than $x$, this has to be a many-to one function, and prior work on 
generative models has tried to define a probabilistic inverse of this function.  Sometimes ---e.g., in context of denoising autoencoders--- such inverses are called {\em reconstruction}
if one thinks of all inverses as being similar.
Here we abandon the idea of reconstruction and  focus on defining {\em many} inverses 
$\tilde{x}$ of $h$. We define a {\em shadow distribution} $P_{W,\rho}$ for
$x|h$, such that a random sample $\tilde{x}$ from this distribution
satisfies Property (b), i.e., $r(W^T\tilde{x}+ b) \approx h$ where $\approx$ denotes approximate equality of vectors.
To understand the considerations in defining such an inverse, 
one must keep in mind that the ultimate goal is to extend the notion 
 to multi-level nets. Thus a generated \textquotedblleft inverse\textquotedblright~$\tilde{x}$ has to look like the {\em output} of a 1-level net in the layer below. As mentioned, this is  where previous attempts such as DBN or denoising autoencoders run into theoretical difficulties. 
%\iffalse For instance, denoising autoencoders ---using so-called {\em weight-tying}---use the {\em deterministic} inversion  $\hat{x} = Wh$, but then $\hat{x}$ is not sparse  and may have negative coordinates, and thus not look like the output of the RELU layer below. The proposal below 
%fixes this in the obvious way. The more difficult part is ensuring property (b), namely, 
%
%
%\Tnote{Start editing here: }
%\fi 

From now on we use $x$ to denote both the random variable and  a specific sample from the distribution defined by the random variable. Given $h$, sampling $x$ from the distribution $P_{W,\rho}$ consists of first computing $r(\alpha Wh)$ for a scalar $\alpha$ and then randomly zeroing-out each coordinate with probability $1-\rho$. (We refer to this noise model  as \textquotedblleft dropout noise.\textquotedblright) Here  $\rho$ can be reduced to make $x$ as sparse as needed; typically $\rho$ will be small. More formally, we have (with $\odot$ denoting  entry-wise product of two vectors):
\begin{equation}
x = r(\alpha Wh)~\odot n_{\drop}\,,\label{eqn:generative_model}
\end{equation}
where $\alpha = 2/(\rho n)$ is a scaling factor, and $n_{\drop}\in \{0,1\}^n$ is a binary random vector with following probability distribution where ($\|n_{\drop}\|_0$ denotes the number of non-zeros of $n_{\drop}$),  
\begin{equation}
\Pr[n_{\drop}] = \rho^{\|n_{\drop}\|_0}.
\end{equation}

Model~\eqref{eqn:generative_model} defines the conditional probability $\Pr[x\vert h]$ (Property (a)). 
The  next informal claim
(made precise in Section~\ref{sec:singlereverse}) shows that Property (b) also holds, 
%ustifies this, 
provided the net (i.e., W) is random-like, and 
%Regarding the distribution of $h$, we only assume 
$h$ is sparse and nonnegative (for precise statement see Theorem~\ref{thm:single}).

\iffalse 
 This alone is sufficient to prove that the single layer net is reversible, that is, Property (b) is satisfied for the joint distribution of $(h,x)$ described above. 
 \fi

\begin{theorem}[Informal version of Theorem~\ref{thm:single}]\label{thm:informal}
	If entries of $W$ are drawn from i.i.d Gaussian prior, then for $x$ that is generated from  model~\eqref{eqn:generative_model}, there is a threshold $\theta \in \R$ such that $r(W^T x + \theta \one) \approx h$, where
	$\one$ is the all-$1$'s vector.
\end{theorem}

Section~\ref{sec:multilayer} shows that this  1-layer generative model composes layer-wise while preserving Property (b). The main reason this works is that
the above theorem makes  minimal distributional assumptions about  $h$. 
(In contrast with RBMs where $h$ has a complicated distribution that is difficult to match to the next layer.)

\subsection{Why random-like nets hypothesis helps theory}
\label{sec:whyrandom}

\iffalse 
Then the column vectors of $W^T$ are roughly orthogonal 
(two random unit vectors in $\Re^m$ have inner product of the order of $1/\sqrt{m}$) and so $W W^T$ is roughly like an identity matrix
with $1$'s on the diagonals and low values on the off-diagonal entries. 
Then $W$ is an inverse of sorts for $W^T$, and 
allows recovery of $x$ from $h$ {\em provided $x$ is a sufficiently sparse vector}.
% Thus thinking of inversion as reconstruction seems
But such strong constraints on sparsity etc. don't hold in real-life neural nets.

\fi

%To explain the {\em random-like net hypothesis} and why it implies a generative analog of the deep net. Along the way we encounter
%denoising autoencoders and examine why they did not lead to a generative model.

%\Tnote{change to $\tilde{x}$ because somehow $\hat{}$ might be often used as estimator}

%The starting point of the current paper is that one can use $h, W$ to define a very simple probability distributions on This only uses the %random-like nets assumption, and works for much more general values of sparsities of $h, x$. 

%The basic idea is as follows. 
%\Tnote{Start editing here}
%Suppose we have two layer of variables, where $h\in \mathbb{R}^m$ denotes hidden representation of the layer above and $x\in \mathbb{R}^n$ %denotes the observable variables below. 

 Continuing informally, we explain  why the random-like nets hypothesis ensures that the shadow distribution satisfies Property (b).
 % well-defined. \Tnote{I feel it could be not quite clear for what ``well-define" means. We may need to define what does it mean by "well-define" first if we want to use this term. }
  It's helpful to first consider the {\em linear} (and deterministic) generative model sometimes used in autoencoders,  $\hat{x} = Wh$ --- also a subcase of ~\eqref{eqn:generative_model} where $\rho = 1$ and the rectifier linear $r(\cdot)$ is removed. 

Suppose the entries of $W$ are chosen from a distribution with mean $0$ and variance $1$
and are independent of $h$. %(Remember that we are describing a multilayer generative model, so $h$ is generated using the higher layer, which features different edge weights than $W$). 
We now show $\hat{h} = W^T\hat{x}$  is close to $h$ itself up to a proper scaling\footnote{Since RELU function $r$ satisfies  for any nonnegative $\alpha$, $r(\alpha x) = \alpha r(x)$,  readers should from now on basically ignore the positive constant scaling, without loss of generality.}. 

A simple way to show this would be to write $\hat{h} = W^TWh$, and then use the fact that for random Gaussian matrix $W$, the covariance $W^TW$ is approximately the identity matrix. However, to get better insight we work out the calculation more laboriously  so as to  
later allow an intuitive extension to the nonlinear case.
%Readers might want to keep in mind that the arguments below are supposed to demonstrate intuitions and are written in a way that extension to nonlinear case is possible. %Let $\hat{h} = W^T\hat{x}$ be the hidden vector \textit{before} applying the bias shift and RELU and .  
Specifically, rewrite a single coordinate  $\hat{h}_i$  as a linear combination of $\hat{x}_j$'s: 
\vspace{-.05in}
\begin{equation}
\hat{h}_i = \sum_{j=1}^n W_{ji}\hat{x}_j. \label{eqn:hi}
\end{equation}
By the definition of $\hat{x}_j$, each term on the RHS of~(\ref{eqn:hi}) is equal to 
\vspace{-.05in}
\begin{equation}
W_{ji}\hat{x}_j = W_{ji}\sum_{\ell=1}^m W_{j\ell}h_{\ell} = \underbrace{W_{ji}^2h_i}_{\textrm{signal:}\mu_j} + \underbrace{W_{ji}\sum_{\ell\neq j}W_{j\ell}h_{\ell}}_{\textrm{noise:}\eta_j}\,.\label{eqn:wjihatxj}
\end{equation}
We split the sum this way to highlight that the first subgroup of terms reinforce each other since 
%of $W_{ji}x_j$ to $\hat{h}_i$.  The term
 $\mu_j = W_{ji}^2 h_i$ is always nonnegative. Thus  $\sum_j \mu_j$ accumulates into a large multiple of $h_i$ and becomes
 the main \textquotedblleft signal\textquotedblright~on the the RHS of~\eqref{eqn:hi}. On the other hand, the noise term $\eta_j$, though it would typically dominate $\mu_j$ in~\eqref{eqn:wjihatxj} in magnitude, has mean zero. Moreover, $\eta_j$ and $\eta_{t}$ for different $j,t$ should be rather independent since they depend on different set of $W_{ij}$'s (though they both depend on $h$). When we sum equation~\eqref{eqn:wjihatxj} over $j$, the term  
 $\sum_j \eta_j$ attenuates due to such cancellation. 
 %and obtain the RHS of~\eqref{eqn:hi} as a sum of $b_j$' and $\sigma_j$'s, %where $b_j$' accumulates and $\sigma_j$
%Given the definition above we can write~\eqref{eqn:hi} as
\vspace{-.1in}
\begin{equation}
\hat{h}_i = \sum_{j=1}^n \mu_j + \sum_{j=1}^n \eta_j \approx nh_i + R\label{eqn:bias-noise}
\end{equation}
where we used that $\sum_{j}W_{ji}^2  \approx n$, and $R = \sum_j \eta_j$. Since $R$ is a sum of random-ish numbers, due to the averaging effect, it scales as $\sqrt{nm}$. For large $n$ (that is $\gg m$), the signal dominates the noise and we obtain $\hat{h}_i \approx nh_i$. %Moreover, if we chose $b' \approx -\sqrt{n\log n}$ so that $|b'| > R$ with high probability. Then we obtain that for $i$ with $h_i = 0$, $\tilde{h}_i = r(\hat{h}_i + b) = r(nh_i + R + b)= nh_i$. Similarly we can show that for non-zero $h_i$, $\tilde{h}_i = nh_i \pm O(|b|)$. It follows that $\|\frac{1}{n}\hat{h} - h\|\lesssim \frac{1}{\sqrt{n}}$. 
This argument can be made rigorous to give  the proposition below, whose  proof appears in Section~\ref{app:complet-one-layer}.

\begin{proposition}[Linear generative model]
	Suppose entries $W$ are i.i.d Gaussian. Let $n > m$, and $h \in \{0,1\}^m$, and $\hat{x}$ is generated from the deterministic linear model $x = Wh$, then with high probability,  the recovered hidden variable 
$\hat{h} = W^Tx $ satisfies that $\|\hat{h} - h\|_{\infty} \le \tilde{O}\left(\sqrt{m/n}\right)$
\end{proposition}

%$C = \sum_{j}W_{ji}^2 > 0$ is a constant, and since $\sigma_j$ are mean zero and random-like, we expect the$R = \sum_j \sigma_j$ to be small. 
The biggest problem with this simple linear generative model is that signal  dominates noise in ~\eqref{eqn:bias-noise} only when  $n \gg m$. That is, the feed-forward direction must always reduce dimensionality by a large
factor, say $10$. This is unrealistic, and fixed in the nonlinear model,   with RELU gates playing an important \textquotedblleft denoising\textquotedblright role.

\vspace{-2mm}
\subsection{\bf Formal proof of Single-layer Reversibility}
\label{sec:singlereverse}

%We fix the problems with the above linear inversion with a nonlinear inversion that uses RELUs and
%dropout noise. 
%\iffalse In this section we formally state and prove Theorem~\ref{thm:informal} and its implications under the generative model~\eqref{eqn:generative_model} that describes $\Pr[x\vert h]$. We prove why despite being produced through a nonlinear function $r$ and entry-wise dot product with $n_{\drop}$, the resulting $x$ still enjoys the desired property that $h \approx \relu(W^Tx+b)$ for some properly chosen $b$. 
%\fi 
Now we give a formal version of Theorem~\ref{thm:informal}. We begin by introducing a succinct notation for model~\eqref{eqn:generative_model}, which will be helpful for later sections as well. 
Let $t = \rho n$ be the expected number of non-zeros in the vector $n_{\drop}$, and let $s_t(\cdot)$ be the random function  that drops coordinates with probability $1-\rho$, that is, $s_t(z) =  z \odot n_{\drop}$. Then we rewrite model~\eqref{eqn:generative_model} as 
\begin{equation}
x = s_t(r(\alpha Wh)) \label{eqn:succint}
\end{equation}
%Let $k$ be such that  $h$ is $k$-sparse. 
We make no assumptions on how large  $m$ and $n$ are except to require that $k < t$. Since $t$ is the (expected) number of non-zeros in vector $x$, it is roughly ---up to logarithmic factor--- the amount of \textquotedblleft information\textquotedblright in $x$. Therefore the assumption that $k < t$ is in accord with the usual \textquotedblleft bottleneck\textquotedblright\  intuition, which says that the representation at higher levels has fewer bits of information.

The random-like net hypothesis here  means that the entries of $W$ independently satisfy
$W_{ij}\sim \mathcal{N}(0,1)$.
%We formalize the calculation above in the following theorem. Mathematically to make the argument above formally, we need to prove that sum of $\sigma_j$'s behaves like a gaussian for most of $h$. We assume that the entries of $W$ is drawn from a random gausain prior, that is, 
\begin{equation}
W_{ij}\sim \mathcal{N}(0,1) \label{eqn:prior}
\end{equation}
%where $W_{ij}$ are independent with each other. 
%Then we have the following theorem
%Then we assume that $h$ has a distribution $\mathcal{D}_h$ over $k$-sparse vectors, 
Also we assume the hidden variable $h$ comes from any distribution $D_h$ that is supported on nonnegative $k$-sparse vectors for some $k$,  where none of the nonzero coordinates are very dominant.
(Allowing a few large coordinates wouldn't kill the theory but the notation and math gets hairier.) 
Technically, when $h\sim D_h$, 
\begin{eqnarray}
h \in \mathbb{R}_{\ge 0}^n, \quad |h|_0 \le k \,\textrm{ and }\, |h|_{\infty} \le \beta\cdot \|h\|\quad \textrm{almost surely} \label{eqn:hdistribution}
\end{eqnarray}
where $\beta = O(\sqrt{(\log k)/k}\ )$. The last assumption essentially says that all the coordinates of $h$ shouldn't be too much larger than the average (which is $1/\sqrt{k} \cdot \|h\|$). 
As mentioned earlier the weak assumption on $D_h$ is the key to  layerwise composability.
\begin{theorem}[Reversibility]\label{thm:single}
	Suppose $t = \rho n$ and $k$ satisfy that $k < t < k^2$. 
	For $(1-n^{-5})$ measure of $W$'s, there exists offset vector $b$, such that the following holds:
	%Then there exists offset vector $b$, such that for a typical\footnote{By this we mean that with high probability over the randomness of the prior measure of $W$, the statement of the theorem is true} $W$ from the prior~\eqref{eqn:prior}, 
	when $h\sim D_h$ and $\Pr[x\vert h]$  is specified by model~\eqref{eqn:generative_model}, then with high probability over the choice of $(h,x)$, 
%	\begin{equation}
%	\Pr\left[\|r(W^Tx+b) - h\|^2 \le \tilde{O}(k/(\rho n))\cdot \|h\|\right] \ge 1 - 1/\poly(n)
%	\end{equation}
	\begin{equation}
\|r(W^Tx+b) - h\|^2 \le \widetilde{O}(k/t)\cdot \|h\|^2,.\label{eqn:reversible}
	\end{equation}
	%\|h-\tilde{h}\|^2 \le \widetilde{O}(k/t)\cdot \|h\|\label{eqn:reversible}
	%	
%	for $h\sim D_h$ that satisfies~\eqref{eqn:hdistribution}, and $x = s_t(r(\alpha Wh))$ drawn from generative model~\eqref{eqn:generative_model}, and $k < t < k^2$, %we have that for most of the $k$-sparse $h$, 
%	we have that with high probability over the randomness of $h$ and $x$, 
%	the vector $\tilde{h} = r(W^Tx + b)$ satisfies
%	$
%	\|h-\tilde{h}\|^2 \le \widetilde{O}(k/t)\cdot \|h\|\label{eqn:reversible}
%	$. 
	%For any fixed $h$ with $|h|_0 = k\ll t$ and $\|h\|_{\infty}\ll \|h\|_2$, for $\alpha = \frac{2}{t}$ and some $b$ with $\tilOmega(\frac{\|h\|}{\sqrt{t}})\le b < \frac{1}{2}$, with high probability over the randomness of $W$, $\tilde{h}$ satisfies that $$\|h-\tilde{h}\|^2\le \tilO\left(\frac{k}{t}\right)\|h\|^2$$ %for $\hat{h} = \alpha W^Tx$  with the scaling constant $\alpha = \frac{2}{s}$, we have that $|\hat{h}-h|_{\infty} \le \tilO\left(\frac{\|h\|}{\sqrt{s}}
	%\right)+ \tilO\left(\frac{1}{\|h\|^3}\right)$
\end{theorem} 
%\begin{theorem}[Reversibility]\label{thm:single}
%	There exists offset vector $b$, such that for a typical\footnote{By this we mean that with high probability over the randomness of the prior measure of $W$, the statement of the theorem is true} $W$ from the prior~\eqref{eqn:prior}, for $h\sim D_h$ that satisfies~\eqref{eqn:hdistribution}, and $x = s_t(r(\alpha Wh))$ drawn from generative model~\eqref{eqn:generative_model}, and $k < t < k^2$, %we have that for most of the $k$-sparse $h$, 
%	we have that with high probability over the randomness of $h$ and $x$, 
%	the vector $\tilde{h} = r(W^Tx + b)$ satisfies
%	$
%		\|h-\tilde{h}\|^2 \le \widetilde{O}(k/t)\cdot \|h\|\label{eqn:reversible}
%	$. 
%	%For any fixed $h$ with $|h|_0 = k\ll t$ and $\|h\|_{\infty}\ll \|h\|_2$, for $\alpha = \frac{2}{t}$ and some $b$ with $\tilOmega(\frac{\|h\|}{\sqrt{t}})\le b < \frac{1}{2}$, with high probability over the randomness of $W$, $\tilde{h}$ satisfies that $$\|h-\tilde{h}\|^2\le \tilO\left(\frac{k}{t}\right)\|h\|^2$$ %for $\hat{h} = \alpha W^Tx$  with the scaling constant $\alpha = \frac{2}{s}$, we have that $|\hat{h}-h|_{\infty} \le \tilO\left(\frac{\|h\|}{\sqrt{s}}
%	%\right)+ \tilO\left(\frac{1}{\|h\|^3}\right)$
%\end{theorem} 
%\Snote{Why not say whp over choice of W as well? }

%\Tnote{Sanjeev, could you see if the phrasing of the theorem okay? It is not very rigourous since I am using  ``for a typical $W$ from the prior" }
The next theorem says that the generative model predicts that the trained net should be stable to dropout. 

\begin{theorem}[Dropout Robustness]\label{thm:dropout}
	Under the condition of Theorem~\ref{thm:single}, suppose we further drop 1/2 fraction of values of $x$ randomly and obtain $x^{\textrm{drop}}$, then there exists some offset vector $b'$, such that with high probability over the randomness of $(h,x^{\drop})$, %$x$ and $x^{\textrm{drop}}$, 
	%for $\tilde{h}^{\textrm{drop}} = r(2W^Tx^{\textrm{drop}} + b')$
%	we have 
%		$\|h-\tilde{h}^{\textrm{drop}} \|^2 \le \tilO(k/t)\cdot \|h\|$.
we have 
\begin{equation}
\|r(2W^Tx^{\textrm{drop}} + b')- h\|^2 \le \tilO(k/t)\cdot \|h\|^2.
\end{equation}

\end{theorem}
%\Tnote{ToDo: needs to check whether we say $t = \rho n$ is the expected sparsity of $x$ somewhere. }

To parse the theorems, we note that the error is on the order of the sparsity ratio $k/t$ (up to logarithmic factors), which is necessary since information theoretically the generative direction must increase the amount of information so that good inference is possible. In other words, our theorem shows that under our generative model, feedforward calculation (with or without dropout) can estimate the hidden variable up to a small error when $k \ll t$. Note that this is 
%this doesn't necessarily mean that we can solve
different from usual notions in generative models such as MAP or MLE inference. We get direct guarantees on the estimation error which neither MAP or MLE can guarantee\footnote{Under the situation where MLE (or MAP) is close to the true $h$, our neural-net inference algorithm is also guaranteed to be the close to both MLE (or MAP) and the true $h$.}. %solves the inference problem (even with dropout) approximately when $k\ll t$.  %the error between feedforward representation $\tilde{h}$ (or the dropout version $\tilde{h}^{\drop})$) and the true hidden variable is 

%\Snote{Maybe address the NIPS reviewer question about what is the meaning of solving the inference problem. MAP etc.}

Furthermore, in Theorem~\ref{thm:dropout}, we need to choose a scaling of factor 2 to compensate the fact that  half of the signal $x$ was dropped. Moreover, we remark that an interesting feature of our theorems is that we show an almost uniform offset vector $b$ (that is, $b$ has almost the same entries across coordinate) is enough. This matches the observation  that in the trained Alex net~\cite{ImageNet}, for most of the layers the offset vectors are almost uniform\footnote{Concretely, for the simple network we trained using implementation of~\cite{jia2014caffe}, for 5 out of 7 hidden layers, the offset vector is almost uniform with mean to standard deviation ratio larger than 5. For layer 1, the ratio is about 1.5 and for layer 3 the offset vector is almost 0.}. In our experiments (Section~\ref{sec:exp}), we also found restricting the bias terms in RELU gates to all be some constant makes almost no change to the performance. 
%\Tnote{Need to decide whether to amphasize this}
%\Tnote{Edited here. }

%\Tnote{A tentative sketch}
%Now we provide the proof for Theorem~\ref{thm:single}. 
We devote the rest of the sections to sketching a proof of Theorems~\ref{thm:single} and~\ref{thm:dropout}. 
The main intermediate step of the proof is the following lemma, which will be proved at the end of the section: 

\begin{lemma}\label{lem:beforerelu}
	Under the same setting as Theorem~\ref{thm:single}, with high probability over the choice of $h$ and $x$, we have that for $\delta = \widetilde{O}(1/\sqrt{t})$, 
	\begin{equation}
	\|W^Tx - h \|_{\infty} \le \delta \|h\|\label{eqn:infinitenorm}\,.
	\end{equation}
\end{lemma}

Observe that since $h$ is $k$-sparse, the average absolute value of the non-zero coordinates of $h$ is $\tau = \|h\|/\sqrt{k}$. Therefore the entry-wise error $\delta\|h\|$ on RHS of~\eqref{eqn:infinitenorm} satisfies $\delta\|h\|= \epsilon \tau$ with $\epsilon = \widetilde{O}(\sqrt{k/t})$. %That is, the hidden state before offset and RELU, $W^Tx$,  is close to $h$ with entry-wise error  $\epsilon \tau$. T
This means that $W^Th$ (the hidden state before offset and RELU) estimates the non-zeros of $h$ with entry-wise relative error $\epsilon$ (in average), though the relative errors on the zero entries of $h$ are huge. Therefore, even though we get a good estimate of $h$ in $\ell_{\infty}$ norm, the $\ell_2$ norm of the difference $W^Tx -h$ could be as large as $\epsilon \tau \sqrt{n}  = \widetilde{O}(\sqrt{m/t})\|h\|$ which could be larger than $\|h\|$. % ($\ell_2$ norm error is needed for )

It turns out that the offset vector $b$ and RELU have the denoising effects that reduce errors on the non-support of $h$ down to 0 and therefore drive the $\ell_2$ error down significantly. The following proof of Theorem~\ref{thm:single} (using Lemma~\ref{lem:beforerelu}) formalizes this intuition. 
	\vspace{-.1in}
\begin{proof}[Proof of Theorem~\ref{thm:single}]
		Let $\hat{h} = W^Tx$. By Lemma~\ref{lem:beforerelu}, we have that $|\hat{h}_i - h_i| \le \delta \|h\|$ for any $i$. Let $b = -\delta\|h\|$. For any $i$ with $h_i = 0$, $\hat{h}_i + b \le h_i + \delta \|h\| + b \le 0$, and therefore $r(\hat{h}_i + b ) = 0 = h_i$. On the other hand, for any $i$ with $h_i \neq 0$, we have $|\hat{h}_i + b  - h_i| \le |b| + |\hat{h}_i - h_i|\le 2\delta \|h\|$. It follows that  $|r(\hat{h}_i + b)  - r(h_i)| \le 2\delta \|h\|$. Since $h_i$ is nonnegative, we have that $|r(\hat{h}_i + b)  - h_i| \le 2\delta \|h\|$. Therefore the $\ell_2$ distance between $r(W^Tx+b) = r(\hat{h}+b)$ and $h$ can be bounded by 
		%To finish, we show that taking another rectifier linear on $\alpha\hat{h}$ will further reduce the noise. Suppose $h$ has support $K$ of size $k$. We choose $\tilde{b} = -\epsilon\mathbf{1}$, and let $\tilde{h}  =\relu(\hat{h} +  \tilde{b}) = \relu(W^Th + \tilde{b})$, then we obtain that for $i\not \in K$,  $\tilde{h}_i  = r(\hat{h}_i ) = r(h_i\pm \epsilon + b) = 0$. Therefore, for those $i\not\in K$, there is no loss $\tilde{h}_i = h_i  = 0$. Moreover, this addition nonlinear operation doesn't incur more loss on $K$ either -- for any $i$, it is guaranteed that $|\tilde{h}_i - h_i| \le |\tilde{h_i}-\hat{h}_i|+|\hat{h}_i- h_i|\le 2\epsilon$. Therefore the total error between $\tilde{h}$ and $h$ in euclidean distance can be bounded by %Therefore using rectifier linear with a small shift $b \approx \epsilon$, the resulting $\tilde{h}$ is exactly equal to 0 when $h_i = 0$, and for $h_i \neq 0$, it only introduce at most another small error $\epsilon$. Therefore then we have that 
		%\begin{equation}
		$\|r(\hat{h}+b) - h\|^2 = \sum_{i : h_i \neq 0} |r(\hat{h}_i + b)  - h_i|^2 \le 4k \delta^2 \|h\|^2$. Plugging in $\delta = \widetilde{O}(1/\sqrt{t})$ we obtain the desired result. 
\end{proof}
Theorem~\ref{thm:dropout} is a direct consequence of Theorem~\ref{thm:single}, essentially due to the dropout nature of our generative model (we sample in the generative model which connects to dropout naturally). See Section~\ref{app:complet-one-layer} for its proof. 
We conclude with a proof sketch  of Lemma~\ref{lem:beforerelu}. The complete proof can be found in Section~\ref{app:complet-one-layer}. 
	\vspace{-.1in}
\begin{proof}[Proof Sketch of Lemma~\ref{lem:beforerelu}]

	Here due to page limit, we give a high-level proof sketch that demonstrates the key intuition behind the proof, by skipping most of the tedious parts. See section~\ref{app:complet-one-layer} for full details. 
	
	By a standard Markov argument, we claim that it suffices to show that w.h.p over the choice of ($(W,h,x)$ the network is reversible, that is, 
	$
	\Pr_{x,h,W}\left[\textrm{equation~\eqref{eqn:reversible} holds}\right] \ge 1 - n^{10}\,.
	$
	%Indeed, suppose~\eqref{eqn:joint-prob} is true, then using a standard Markov argument, we obtain  
	%$$\Pr_{W}\left[\Pr_{x,h}\left[\textrm{equation~\eqref{eqn:reversible} holds} \mid W\right] > 1-n^{-5} \right]\ge 1 - n^{5},$$ as desired. 
	%Equation~(\ref{eqn:joint-prob}) is proven rigorously in Theorem~\ref{thm:single-technical} in supplementary materials, though here we give a proof sketch that demonstrates the idea. 
	Towards establishing this, we note that equation~\eqref{eqn:reversible} %(and consequently statement~\eqref{eqn:joint-prob}) 
	holds for any positive scaling of $h,x$ simultaneously, since RELU has the property that $r(\beta z) = \beta \cdot r(z)$ for any $\beta \ge 0$ and any $z\in \mathbb{R}$.  Therefore WLOG we can choose a proper scaling of $h$ that is convenient to us. We assume $\|h\|_2^2 = k$. By assumption~\eqref{eqn:hdistribution}, we have that $|h|_{\infty}\le \widetilde{O}(\sqrt{\log k})$. 
	
	Define $\hat{h} = \alpha W^Tx$. Similarly to~\eqref{eqn:hi}, we fix $i$ and expand the expression for $\hat{h}_i$ by definition, 
		\vspace{-.05in}
	\begin{equation}
	\hat{h}_i  = \alpha\sum_{j=1}^n W_{ji}x_j 
	\end{equation}	
	Suppose $n_{\drop}$ has support $T$.  Then we can write $W_{ji}x_j$ as 
	\vspace{-.05in}
	\begin{align}
	W_{ji}x_j  &=  W_{ji}\relu(\sum_{\ell=1}^m W_{j\ell} h_{\ell}) \cdot n_{\drop,j}=  W_{ji}\relu(W_{ji} h_i + \eta_j)\cdot \mathbf{1}_{j\in T} \,,\label{eqn:eqn29-main}
	\end{align}
	where $\eta_j \triangleq \sum_{\ell\neq j}W_{j\ell}h_{\ell}$. Though $\relu(\cdot)$ is nonlinear, it is piece-wise linear and more importantly still Lipschitz. Therefore intuitively, RHS of~\eqref{eqn:eqn29-main} can be ``linearized" by approximating 
	$	\relu(W_{ji}h_i + \eta_j) $ by $\mathbf{1}_{\eta_j > 0}\cdot \left(W_{ji}h_i + \relu(\eta_j)\right)$. 
	%	\begin{equation}
	%	\relu(W_{ji}h_i + \eta_j) \approx  \mathbf{1}_{\eta_j > 0}\cdot W_{ji}h_i + \relu(\eta_j) \label{eqn:approx}
	%	\end{equation}
	Note that this approximation is not accurate only when $|\eta_j|\le |W_{ji}h_i|$, which happens with relatively small probability, since $\eta_j$ typically dominates $W_{ji}h_{i}$ in magnitude. 
	
	%Under this approximation, we found that 

	Then $W_{ji}x_j$ can be written approximately as $\mathbf{1}_{\eta_j > 0}\cdot \left(W_{ji}^2h_i + W_{ji}\relu(\eta_j)\right)$, where the first term corresponds to the bias and the second one corresponds to the noise (variance), similarly to the argument in the linear generative model case in Section~\ref{sec:whyrandom}. 
	
	Using the intuition above, one can formally prove (as in the full proof in Section~\ref{app:complet-one-layer}) that $\Exp[W_{ji}x_j] =\frac{1}{2} h_i \pm \widetilde{O}(1/k^{3/2})$ and $\Var[W_{ji}x_j] = O(\sqrt{k/t})$. By a concentration inequality\footnote{Note that to apply concentration inequality, independence between random variables is needed. Therefore technically, we need to condition on $h$ and $T$ before applying concentration inequality, as is done in the full proof of this lemma in Section~\ref{app:complet-one-layer}. } for the sum  $\hat{h}_i = \alpha\sum_{j=1}^n W_{ji}x_j $, one can show that 		
	with high probability ($1-n^{-10}$) we have $|\hat{h}_i -  h_i| \le \widetilde{O}(\sqrt{k/t})$.  %, where the first term account for the variance and the second is due to the bias (difference between $\Exp[\hat{h}_i]$ and $h_i$). 
		%$$\Pr\left[|\hat{h}_i -  h_i| \ge \widetilde{\Omega}((\sqrt{k/t}+ 1/(tk^{3/2}))\mid h\right] \le n^{-10}\,.$$		
		%Noting that $k < t < k^2$ and therefore $1/(tk^{3/2})$ is dominated by $\sqrt{k/t}$, take expectation over $h$, we obtain that $\Pr\left[|\hat{h}_i -  h_i| \ge \widetilde{\Omega}(\sqrt{k/t})\right] \le n^{-10}$. 
		Taking union bound over all $i\in [n]$, we obtain that with high probability, $\|\hat{h} - h\|_{\infty} \le \widetilde{O}(\sqrt{k/t})$.  Recall that $\|h\|$ was assumed to be equal to $k$ (WLOG). Hence we obtain that $\|\hat{h} - h\|_{\infty} \le \widetilde{O}(\sqrt{1/t})\|h\|$ as desired. 
\end{proof}
%Therefore $\tilde{x}$ is still nonnegative so that potentially we can compose another layer below. 

%We see that $W_{ji}\hat{x}_j$ contains 

%Our generative model works as follows:  Given $h$, suppose we generate $\tilde{x} = W^Th$ 
%we first compute $\hat{x} = \relu(W^Tx+b)$

%Suppose for some $\tilde{x}$ it is true that $r(W\tilde{x} + b) = h$. We fix some index $i$ such that $h_i > 0$. 
\iffalse 
But there is a an inverse of sorts exists if $W$ is random-like. To see this first imagine the mapping is linear, $h = Wx$, and

Suppose  and the columns are unit vectors. Then the columns have small pairwise inner products of around 
$1/\sqrt{m}$ (since random vectors tend to have low inner products) and thus $W^T W$ is an $n \times n$ matrix that has $1$'s on the diagonal and low values around $1/\sqrt{m}$ on the off-diagonals. Thus if $x$ has at most $??$ nonzero entries, 
\fi

\section{Full multilayer model}
\label{sec:multilayer}
\vspace*{-0.1in}
\setlength{\textfloatsep}{15pt}
\begin{figure}\label{generative-discriminative pair}
	\label{fig:gen-disc}
	\hfill
	\subfigure[Generative model]{\includegraphics[width=5cm]{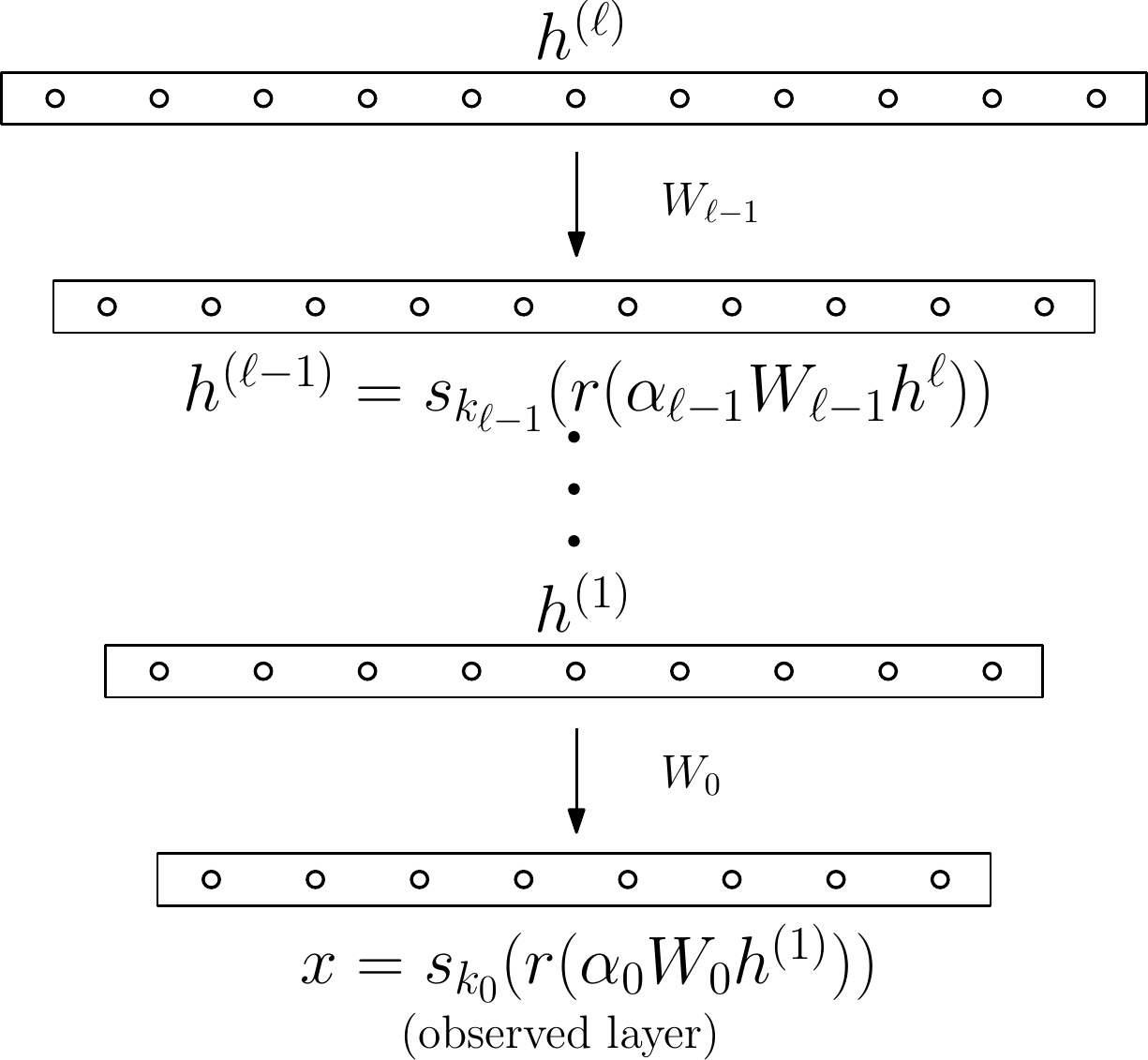}}
	\hfill
	\subfigure[Feedforward NN]{\includegraphics[width=5cm]{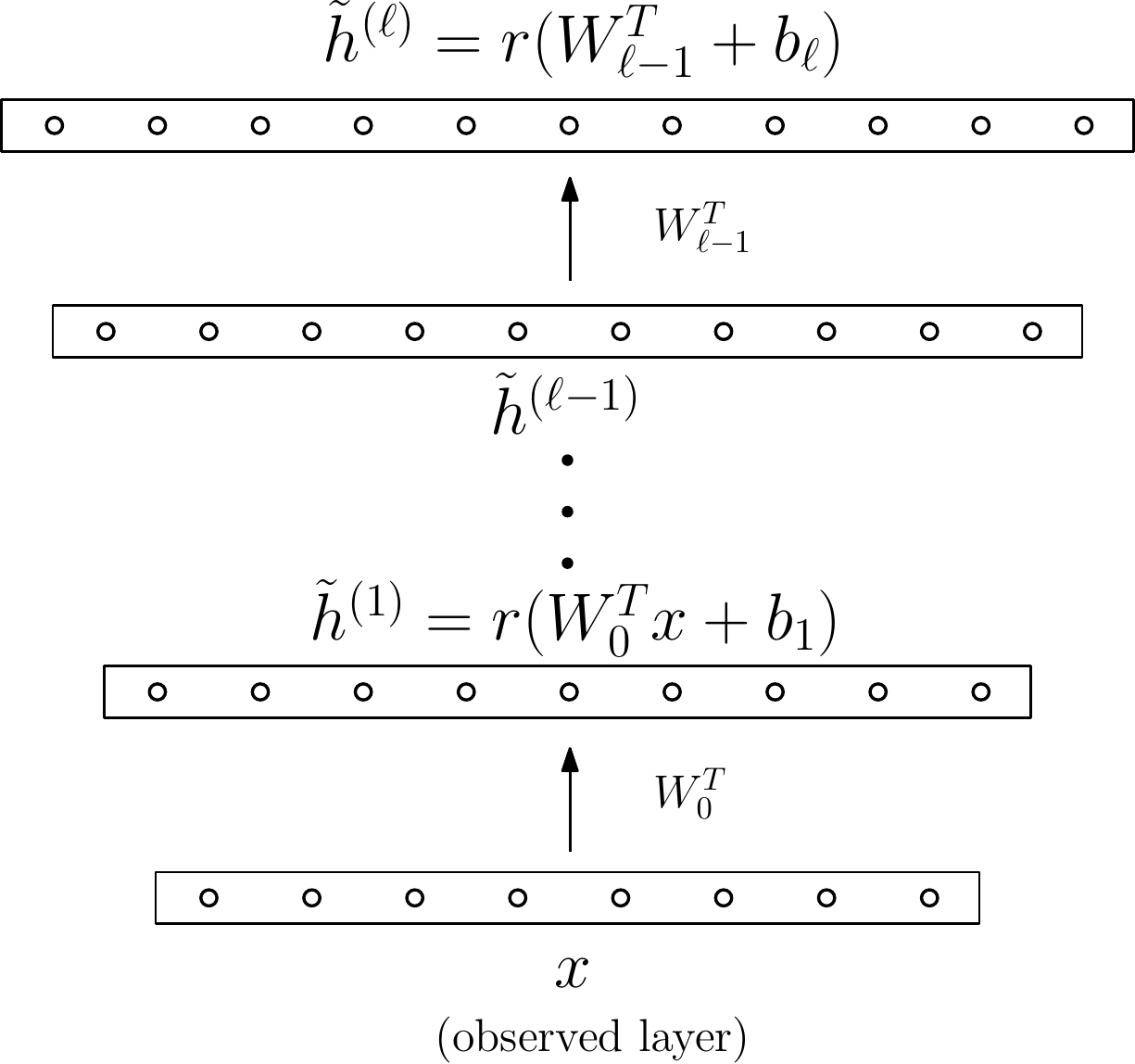}}
	\hfill
	\caption{Generative-discriminative pair: a) defines the conditional distribution $\Pr[x\vert h^{(\ell)}]$; b) defines the feed-forward function $\tilde{h}^{(\ell)} = \textup{NN}(x)$.}
	\end{figure}
%\Tnote{Just witing down the proofs.. no intuition.. minimum explanation.. }

We describe the multilayer generative modelfor the shadow distribution associated with an $\ell$-layer deep net with RELUs. The feed-forward net is
shown in Figure~\ref{fig:gen-disc} (b). The $j$-th layer has $n_j$ nodes, while the observable layer has $n_0$ nodes. The  corresponding generative model is in Figure~\ref{fig:gen-disc} (a).
The number of variables at each layer,  and the edge weights match exactly in the two models, but the generative model runs from top to down.

%{\sc describe model here}
The generative model uses the hidden variables at layer $j$ to produce the hidden variables at $j-1$ using exactly the
single-layer generative analog of the corresponding layer of the deep net. 
It starts with a value $h^{(\ell)}$ of the top layer, which is generative from some arbitrary distribution $D_{\ell}$ over the set of $k_{\ell}$-sparse vectors in $\mathbb{R}^{n_{\ell}}$ (e.g., uniform $k_{\ell}$ subset as support, followed by Gaussian distribution on the support). %(The intent is that this is the distribution produced by applying the feedforward net.) 
Then using weight matrix $W_{\ell-1}$, it generates the hidden variable $h^{(\ell-1)}$ below using the same stochastic process as described for one layer in Section~\ref{sec:model}. Namely,  apply a random sampling function $s_{k_{\ell-1}}(\cdot)$ (or equivalently add dropout noise with  $\rho = 1-k_{\ell-1}/n^{\ell-1}$) on the vector $r(\alpha_{\ell-1}W_{\ell-1}h^{(\ell)})$, where $k_{\ell-1}$ is the target sparsity of $h^{(\ell-1)}$, and $\alpha_{\ell-1} = 2/k_{\ell-1}$ is a scaling constant. Repeating the same stochastic process (with weight $W_j$, random sampling function $s_{k_j}(\cdot)$), we generate $x$ at the bottom layer. In formula, we have 
\begin{equation}
x = s_{k_0}(r(\alpha_0W_0s_{k_1}(r(\alpha_1W_1\cdots \cdot )))\label{eqn:mult-generative}
\end{equation}%We assume that $\ell-1$

%\ffigureh{generativemodel}{2.5in}{Generative Model}{fig2}

%\newcommand{\ffigureh}[4]{\begin{figure}[!h] 
%		\begin{center}  
%			\includegraphics*[height=#2]{#1}
%		\end{center} 
%		\caption{#3}\label{#4} 
%	\end{figure}} 

We can prove Property (b) (that the feedforward net inverts the generative model) formally for 2 layers by adapting the proof for a single layer. The proof for 3 layers works under more restricted conditions. % the deep net inverts the distribution produced by the generative model.
 For more layers the correctness seems
intuitive too but a formal proof seems to require proving concentration inequalities for a fairly complicated and non-linear function. We have verified empirically that Property (b) holds for up to $6$ layers.  %{\sc is this correct??}
%\Tnote{This is intuitively correct though we don't have a concrete open probem in mind .. Maybe we can say that it requires to prove concentration inequalities for highly non-linear system}

We assume the random-like matrices $W_j$'s to have standard gaussian prior as in~\eqref{eqn:prior}, that is, 
\begin{equation}
W_j \textrm{ has i.i.d entries from }\mathcal{N}(0,1)  \label{eqn:multi-prior}
\end{equation}We also assume that the distribution $D_{\ell}$ produces $k_{\ell}$-sparse vectors with not too large entries almost surely as in~\eqref{eqn:hdistribution}, that is, 
$h^{(\ell)} \in \R_{\ge 0}^{n_{\ell}}\,, \,|h^{(\ell)}|_0 \le k_{\ell} \textrm{ and } |h^{(\ell)}|_{\infty} \le O\left(\sqrt{\log N/(k_{\ell})}\right)\|h\|\quad \textrm{a.s.}$ (Note that $N \triangleq \sum_{j}n_{j}$ is the total number of nodes in the architecture. )

Under this mathematical setup, we prove the following reversibility and dropout robustness results, which are the 2-layers analog of Theorem~\ref{thm:single} and Theorem~\ref{thm:dropout}. %Theorem~\ref{thm:twolayer} and Theorem~\ref{thm:twolayer-dropout}. 
\begin{theorem}[2-Layer Reversibility and Dropout Robustness]\label{thm:twolayer}
	For $\ell = 2$, and $k_{2} < k_{1} < k_0 < k_2^2$, for 0.9 measure of the weights $(W_0,W_1)$, the following holds: There exists constant offset vector $b_0, b_1$ such that when $h^{(2)}\sim D_2$ %that satisfies~\eqref{eqn:multi-hdistribution},
	 and $\Pr[x\mid h^{(2)}]$ is specified as model~\eqref{eqn:mult-generative}, then network has reversibility and dropout robustness in the sense that the feedforward calculation (defined in Figure~\ref{generative-discriminative pair}b) gives $\tilde{h}^{(2)}$ satisfying  
%	For $\ell = 2$, and $k_{2} < k_{1} < k_0 < k_2^2$, there exists constant offset vector $b_0, b_1$ such that with probability .9 over the randomness of the weights $(W_0,W_1)$ from prior~\eqref{eqn:multi-prior}, and $h\sim D$ that satisfies~\eqref{eqn:multi-hdistribution}, and observable $x$ generated from the 2-layer generative model described above, then network has reversibility and dropout robustness in the sense that the feedforward calculation gives $\tilde{h}^{(2)}$ satisfying  %the representation obtained from running the network feedforwardly  
%	$$\tilde{h}^{(1)} = r(W_0^Tx+b_1), \textrm{ and }\tilde{h}^{(2)} = r(W_1^Th^{(1)}+b_2)$$
%	are entry-wise close to the original hidden variables 
	\begin{equation}
	\forall i\in [n_2], \quad \Exp\left[|\tilde{h}_i^{(2)} - h_i^{(2)}|^2\right] \le \epsilon \tau^2\label{eqn:entry-wise-errorh2}
	\end{equation}
%	\begin{equation}
%	\forall i\in [n_1], \quad \Exp\left[|\tilde{h}_i^{(1)} - h_i^{(1)}|^2\right] \le \tilO(k_1/k_0)\cdot \bar{h}^{(1)}\label{eqn:entry-wise-errorh1}	
%	\end{equation}
where $\tau= \frac{1}{k_{2}}\sum_{i}h^{(2)}_i$is %and $\bar{h}^{(1)}= \frac{1}{\kappa_{1}}\sum_{i}h^{(1)}_i$ is
 the average of the non-zero entries of $h^{(2)}$ and $\epsilon = \widetilde{O}(k_2/k_1)$. Moreover, the network also enjoys dropout robustness as in Theorem~\ref{thm:dropout}.
%Moreover,  there exists offset vector $b_1'$ and $b_2'$, such that when hidden representation is calculated feedforwardly with dropping out a random subset $G_0$ of size $n_0/2$ in the observable layer and $G_2$ of size $n_1/2$ in the first hidden layer:
%$$\tilde{h}^{(1)\drop} = r(2W_0^Tx_{\bar{G_0}}+b_1'), \textrm{ and }\tilde{h}^{(2)} = r(2W_1^Th^{(1)\drop}_{\bar{G_2}}+b_2')$$
%are entry-wise close to the original hidden variables in the same form as equation~\eqref{eqn:entry-wise-errorh2} and~\eqref{eqn:entry-wise-errorh1}
\end{theorem}
%\begin{theorem}[2-Layer Reversibility]\label{thm:twolayer-dropout}
%%	For $\ell = 2$, and $k_{2} < k_{1} < k_0 < k_2^2$, there exists constant offset vector $b_0, b_1$ such that with probability .9 over the randomness of the weights $(W_0,W_1)$ from prior~\eqref{eqn:multi-prior}, and $h\sim D$ that satisfies~\eqref{eqn:multi-hdistribution}, and observable $x$ generated from the 2-layer generative model described above, the representation obtained from running the network feedforwardly  
%Under the same condition as Theorem~\ref{thm:twolayer}, there exists offset vector $b_1'$ and $b_2'$, such that when hidden representation is calculated feedforwardly with dropping out a random subset $G_0$ of size $n_0/2$ in the observable layer and $G_2$ of size $n_1/2$ in the first hidden layer:
%	$$\tilde{h}^{(1)\drop} = r(2W_0^Tx_{\bar{G_0}}+b_1'), \textrm{ and }\tilde{h}^{(2)} = r(2W_1^Th^{(1)\drop}_{\bar{G_2}}+b_2')$$
%	are entry-wise close to the original hidden variables in the same form as equation~\eqref{eqn:entry-wise-errorh2} and~\eqref{eqn:entry-wise-errorh1}. 
%	%$$\forall i\in [n_2], \quad \Exp\left[|\tilde{h}_i^{(2)\drop} - h_i^{(2)}|^2\right] \le \tilO(k_2/k_1)\cdot \bar{h}^{(2)}$$
%	%$$\forall i\in [n_1], \quad \Exp\left[|\tilde{h}_i^{(1)\drop} - h_i^{(1)}|^2\right] \le \tilO(k_1/k_0)\cdot \bar{h}^{(1)}$$
%	%where $\bar{h}^{(2) }= \frac{1}{\kappa_{2}}\sum_{i}h^{(2)}_i$ and $\bar{h}^{(1)}= \frac{1}{\kappa_{1}}\sum_{i}h^{(1)}_i$ are the average of the non-zero entries of $h^{(2)}$ and $h^{(1)}$. 
%\end{theorem}
To parse the theorem, we note that when $k_2 \ll k_1$ and $k_1\ll k_0$, in expectation, the entry-wise difference between $\tilde{h}^{(2)}$ and $h^{(2)}$ is dominated by the average single strength of $h^{(2)}$. %and so is true for $\tilde{h}^{(1)}$ and $h^{(1)}$. %Moreover, we show that even the feedforward with dropout is robust 

However, we note that we prove weaker results than in Theorem~\ref{thm:single} -- though the magnitudes of the error in both Theorems are on the order of the ratio of the sparsities between two layers, here only the expectation of the error is bounded, while Theorem~\ref{thm:single} is a high probability result. In general we believe high probability bounds hold for any constant layers networks, but seemingly proving that requires advanced tools for understanding concentration properties of a complex non-linear function. Just to get a sense of the difficulties of results like Theorem~\ref{thm:twolayer}, one can observe that to obtain $\tilde{h}^{(2)}$ from $h^{(2)}$ the network needs to be run twice in both directions with RELU and sampling. This makes the dependency of $\tilde{h}^{(2)}
$ on $h^{(2)}$ and $W_2$ and $W_1$ fairly complicated and non-linear. %More importantly, the randomness of $W_1$ and $W_2$ are used in a fairly correlated way which prevents an easy application of concentration inequality.  %and more importantly the dependency involves use the randomness of $h^{(2)}$ and $W_1$ and $W_2$ in a fairly correlated way. %\begin{lemma}

Finally, we extend Theorem~\ref{thm:twolayer} to three layers with stronger assumptions on the sparsity of the top layer -- We assume additionally that $\sqrt{k_3}k_2 < k_0$, which says that the top two layer is significantly sparser than the bottom layer $k_0$. We note that this assumption is still reasonable since in most of practical situations the top layer consists of the labels and therefore is indeed much sparser. Weakening the assumption and getting high probability bounds are left to future study. 

\begin{theorem}[3-layers Reversibility and Dropout Robustness, informally stated]\label{thm:threelayer}
	For $\ell = 3$, when $k_3 < k_2 < k_1 < k_0 < k_2^2$ and  $\sqrt{k_3}k_2 < k_0$, the 3-layer generative model has the same type of reversibility and dropout robustness properties as in Theorem~\ref{thm:twolayer}.
\end{theorem}

%	
%\end{lemma}
%\subsection{Auxiliary lemmas}

%\begin{lemma}
%	For any $a,b \ge 0$ and $\frac{1}{\sigma\sqrt{2\pi}}\exp(-\frac{x^2}{2\sigma^2})$
%\end{lemma}

%\input{reg}

% !TEX root = deepnet_main.tex

\vspace{-2mm}
\section{Experiments}
\label{sec:exp}
\vspace{-2mm}

We present experimental results that support our theory. 
%The experiments focus on verifying the theory rather than the performance. 

\noindent\textbf{Verification of the random-like nets hypothesis.} 
Testing the fully connected layers in a  few trained networks showed 
that the edge weights are indeed random-like. For instance 
 Figure~\ref{fig:verify} in the appendix shows the statistics for the 
%We first show that our assumption is roughly satisfied on neural networks trained in practice. We %took the 
the second fully connected layer in AlexNet (after 60 thousand training iterations).
%, and plotted some statistics in .%(cite)
Edge weights fit a  Gaussian distribution, and bias in the RELU gates are essentially constant (in accord with Theorem~\ref{thm:single}) and 
%re close to each other, mostly within the interval [0.9, 1.1]. 
 the distribution of the singular values of the weight matrix is close to the quartercircular law of random Gaussian matrices.
% Similar behavior was observed in networks other than AlexNet. %cite

\noindent\textbf{Generative model.} 
% out of reasonably well-defined even for a random deep net;
%and training refines it a bit to the true distribution.  
We trained a 3-layer fully connected net on CIFAR-10 dataset as described below.
Given an image and a neural net our shadow distribution can be used to generate an image.
Figure~\ref{fig:syn_images} shows the generated image using the initial (random) deep net, and 
from the trained net after 100,000 iterations.

\iffalse 
 shows three examples 

Each example contains three images: one generated at iteration 0 using random weights, one generated at iteration 100000, and the original image. It can be seen that the image generated after training is more similar to the original image than the one generated by random weights, which aligns with our theory. This also distinguish the random-like network after training from one with random weights: the former has bulk properties similar to the latter, but it is also adjusted to the image distribution by training.
\fi 

\begin{figure}
\vspace{-4mm}
\centering
\subfigure{
	\includegraphics[width=0.15\columnwidth]{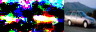} 
}
\hspace{5mm}
\subfigure{
	\includegraphics[width=0.15\columnwidth]{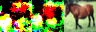}
}
\hspace{5mm}
\subfigure{
	\includegraphics[width=0.15\columnwidth]{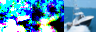}
}
\vspace{-2mm}
\caption{Some synthetic images generated on CIFAR-10 by the shadow distribution. Each subfigure contains three images: the first is generated using the random initial net at iteration 0, the second generated after extensive training to iteration 100,000, and the third is the original image.} 
\label{fig:syn_images}
\end{figure}

\noindent\textbf{Improved training using synthetic data.} 
Our suggestion is to use the shadow distribution to produce 
synthetic images together with labels and add them to the training set (named SHADOW). 
Figure~\ref{fig:syn_images} clarifies  that the generated images are some kind of
{\em noisy} version of the original. Thus the method is reminiscent of dropout, except the \textquotedblleft noise\textquotedblright\ is applied to the original image rather than at  each layer of the net.
% is similar in principle to 
%dropout (which is believed to act as a 
%regularizer).
 We trained feedforward networks with two hidden layers, each containing 5000 hidden units, using the standard datasets MNIST (handwritten digits)~\cite{lecun1998gradient}, and CIFAR-10 and CIFAR-100 (32x32 natural images)~\cite{krizhevsky2009learning}. A similar testbed was recently used in~\cite{NeyshaburSS15}, who were also testing their proposed algorithm for fully connected deep nets.
During each iteration, for each real image one synthetic image is generated from the highest hidden layer using the shadow distribution.  
The loss on the synthetic images is weighted by a regularization parameter. (This parameter and the learning rate were chosen by cross validation; see the appendix for details.) All networks were trained both with and without dropout. When training with dropout, the dropout ratio is 0.5. 

%Figure~\ref{fig:vanilla_reg20} in the appendix shows  test errors with the above training combined with backpropagation alone, and  Figure~\ref{fig:dropout_rd20} shows errors when combined with backpropagation and dropout. 
Figure~\ref{fig:shadow} shows the test errors of networks trained by backpropagation alone,  by backpropagation and dropout, and by our SHADOW regularization combined with backpropagation and dropout. 
Error drops much faster with our training, and  a small but persistent advantage is retained even at the end. In the appendix, we also show that SHADOW outperforms backpropagation alone, and the error on the synthetic data tracks that on the real data, as predicted by the theory.

\iffalse  As expected, our regularization can improve the performance. With the regularization, the final error is smaller compared to the vanilla training method.  Another interesting observation is that the error reach the minimum faster. This implies that the synthetic data generated are effective: they can be used as a surrogate of the true data. The test errors of the networks trained with dropout are shown in   Similarly, the error reaches the minimum significantly faster than using dropout alone. The final error using our regularization is similar to or slightly better than that using dropout.  So our regularization can be combined with dropout to get faster convergence.
\fi 

\begin{figure}
\centering
\vspace{-2mm}
	\includegraphics[width=0.9\columnwidth]{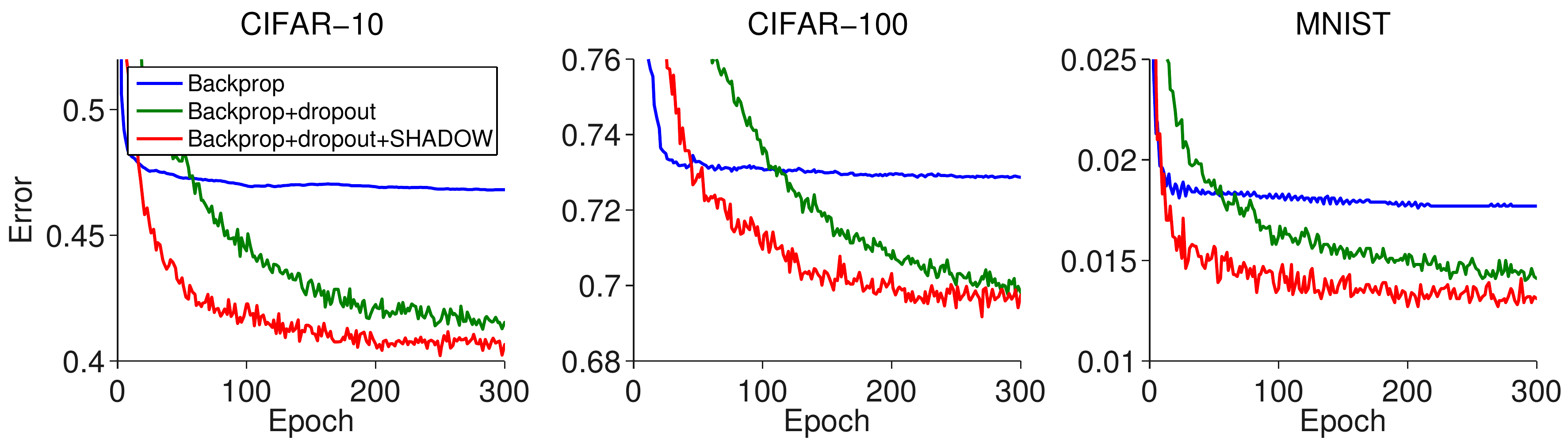}
\caption{Testing error of networks trained with and without our regularization SHADOW for three datasets. } 
\label{fig:shadow}
\end{figure}

\vspace{-.2in}
\section{Conclusions} 
\label{sec:conclu}
\vspace{-.1in}
We have highlighted an interesting empirical finding, that the weights of neural nets obtained by 
standard supervised training behave similarly to random numbers. We have given a mathematical proof that such mathematical properties can lead to a very simple explanation for why neural nets have an associated generative model, and furthermore one that is essentially the reversal of the forward computation with the same edge weights. (This associated model can also be seen as some theoretical explanation of
the empirical success of older ideas in deep nets such as {\em weight tying} and {\em dropout.}) 

The model leads to a natural new modification for training of neural nets with fully-connected layers: in addition to dropout, train also on 
{\em synthetic} data generated from the model. This is shown to give some improvement over plain dropout. Extension of these ideas for convolutional layers is left for future work.

  Theoretically explaining why deep nets satisfy the random-like
nets hypothesis is left for future work. Possibly it follows from some elementary information bottleneck consideration~(\cite{Tishby15}).

Another  problem for future work is to design {\em provably correct} algorithms for deep learning assuming the input is generated {\em exactly} according to our generative model. This could be a first step to deriving provable guarantees on backpropagation.

\bibliographystyle{plainnat}
{\small
	\bibliography{deepnet_main}
	}

\clearpage
\appendix
\begin{center}
\LARGE{\textsc{Appendix}}
\end{center}
\section{Complete proofs for One-layer reversibility}\label{app:complet-one-layer}

\begin{proof}[Proof of Lemma~\ref{lem:beforerelu}] 
	%We fix the choice of $h$ and prove that~\eqref{eqn:reversible} is true for most of $W$ from~\eqref{eqn:prior}. This would implies that 
	We claim that it suffices to show  that w.h.p\footnote{We use w.h.p as a shorthand for ``with high probability'' in the rest of the paper. } over the choice of ($(W,h,x)$ the network is reversible (that is, ~\eqref{eqn:reversible} holds), 
	\begin{equation}
	\Pr_{x,h,W}\left[\textrm{equation~\eqref{eqn:reversible} holds}\right] \ge 1 - n^{10}\,.\label{eqn:joint-prob}
	\end{equation}
	Indeed, suppose~\eqref{eqn:joint-prob} is true, then using a standard Markov argument, we obtain  
	$$\Pr_{W}\left[\Pr_{x,h}\left[\textrm{equation~\eqref{eqn:reversible} holds} \mid W\right] > 1-n^{-5} \right]\ge 1 - n^{5},$$ as desired. 
	
	%Equation~(\ref{eqn:joint-prob}) is proven rigorously in Theorem~\ref{thm:single-technical} in supplementary materials, though here we give a proof sketch that demonstrates the idea. 
	Now we establish~\eqref{eqn:joint-prob}. Note that equation~\eqref{eqn:reversible} (and consequently statement~\eqref{eqn:joint-prob}) 
	holds for any positive scaling of $h,x$ simultaneously, since RELU has the property that $r(\beta z) = \beta \cdot r(z)$ for any $\beta \ge 0$ and any $z\in \mathbb{R}$.  Therefore WLOG we can choose a proper scaling of $h$ that is convenient to us. We assume $\|h\|_2^2 = k$. By assumption~\eqref{eqn:hdistribution}, we have that $|h|_{\infty}\le \widetilde{O}(\sqrt{\log k})$. 
	
	Define $\hat{h} = \alpha W^Tx$. Similarly to~\eqref{eqn:hi}, we fix $i$ and expand the expression for $\hat{h}_i$ by definition and obtain 
	$	\hat{h}_i  = \alpha\sum_{j=1}^n W_{ji}x_j$. 
	%Suppose the random sampling function $s_t(\cdot)$ drops all the values outside support $T$, we can write $x_j$ as 
	Suppose $n_{\drop}$ has support $T$.  Then we can write $x_j$ as 
	\begin{align}
	x_j  &=  \relu(\sum_{\ell=1}^m W_{j\ell} h_{\ell}) \cdot n_{\drop,j}=  \relu(W_{ji} h_i + \eta_j)\cdot \mathbf{1}_{j\in T} \,,\label{eqn:eqn29}
	\end{align}
	where $\eta_j \triangleq \sum_{\ell\neq j}W_{j\ell}h_{\ell}$. Though $\relu(\cdot)$ is nonlinear, it is piece-wise linear and more importantly still Lipschitz. Therefore intuitively, RHS of~\eqref{eqn:eqn29} can be ``linearized" by approximating 
	$	\relu(W_{ji}h_i + \eta_j) $ by $\mathbf{1}_{\eta_j > 0}\cdot \left(W_{ji}h_i + \relu(\eta_j)\right)$. 
	%	\begin{equation}
	%	\relu(W_{ji}h_i + \eta_j) \approx  \mathbf{1}_{\eta_j > 0}\cdot W_{ji}h_i + \relu(\eta_j) \label{eqn:approx}
	%	\end{equation}
	Note that this approximation is not accurate only when $|\eta_j|\le |W_{ji}h_i|$, which happens with relatively small probability, since $\eta_j$ typically dominates $W_{ji}h_{i}$ in magnitude. %(note that  while $W_{ji}h_i$ is only one term while $\eta_j = \sum_{\ell\ne i}W_{j\ell}h_{\ell}$ is a sum, therefore the latter typically dominates the former).
	
	Using the intuition above, we can formally calculate the expectation of $W_{ji}x_j$ via a slightly tighter and more sophisticated argument: Conditioned on $h$, we have $W_{ji}\sim\mathcal{N}(0,1)$ and  $\eta_j = \sum_{\ell\neq j}W_{j\ell}h_{\ell} \sim \mathcal{N}(0,\sigma^2)$ with $\sigma^2 = \|h\|^2 - h_j^2 = k - h_j^2 \ge \Omega(k)$. Therefore we have $\log \sigma = \frac{1}{2}\log (\Omega(k)) \ge h_j$ and then by Lemma~\ref{lem:exp} ($W_{ji}$ here corresponds to $w$ in Lemma~\ref{lem:exp}, $\eta_j$ to $\xi$ and $h_j$ to $h$), we have 
	$
	\Exp[W_{ji}x_j \mid h] = \Exp[W_{ji}r(W_{ji}h_j + \eta_j) \mid h] = \frac{1}{2} h_i \pm \widetilde{O}(1/k^{3/2})\nonumber
	$. 
	
	It follows that 
	\begin{equation}
	\Exp[\hat{h}_i\vert h,T] = \sum_{j\in T} \Exp[W_{ji}r(W_{ji}h_j + \eta_j) \mid h] = \frac{\alpha |T|}{2}  h_i \pm \widetilde{O}(\alpha t/k^{3/2})\nonumber
	\end{equation}
	
	Taking expectation over the choice of $T$ we obtain that \begin{equation}
	\Exp[\hat{h}_i\vert h] = \frac{\alpha t}{2}  h_i \pm \widetilde{O}(\alpha n/k^{3/2}) = h_i \pm \widetilde{O}(1/(k^{3/2})) \label{eqn:expectation}
	\end{equation} . (Recall that $t = \rho n$ and $\alpha = 2/(\rho n)$. Similarly, the variance of $\hat{h}_i\vert h$ can be bounded by 
	$\Var[\hat{h}_i \vert h] = O(k/t)$ (see Claim~\ref{lem:app:var}). Therefore we expect that $\hat{h}_i$ concentrate around its mean with fluctuation $\pm O(\sqrt{k/t})$. Indeed, by 
	%	
	%	Equation~(\ref{eqn:approx}) implies that $W_{ji}x_j$ can be written approximately as 
	%	\begin{align}
	%	W_{ji}x_j % &= W_{ji}\cdot \relu(\sum_{\ell}\cdot W_{j\ell}x_{\ell}) \cdot \mathbf{1}_{j\in T} \nonumber\\
	%	& \approx \underbrace{W_{ji}^2h_i \cdot \mathbf{1}_{\eta_j > 0}\cdot \mathbf{1}_{j\in T}}_{\textrm{signal: } \mu_j'}  +\underbrace{W_{ji}r(\eta_j)\cdot \mathbf{1}_{j\in T} }_{\textrm{noise: } \eta_j'} 
	%	\end{align}
	%Though here we have a more complicated formula than~\eqref{eqn:wjihatxj}, actually the same intuition applies: when taking sum of $W_{ji}x_j$ over $j$, $u_j'$ contributes to the signal term, while the sum of $W_{ji}r(\eta_j')$ attenuates due to cancellation (note that it is still mean zero due to the fact $W_{ji}$ is mean zero and independent with $r(\eta_j')$. %has mean zero and contributes the variance, and we obtain that 
	by concentration inequality, we can prove (see Lemma~\ref{lem:concentration} in Section~\ref{app:complet-one-layer}) that actually $\hat{h}_i\mid h$ is sub-Gaussian, and therefore  conditioned on $h$, with high probability ($1-n^{-10}$) we have $|\hat{h}_i -  h_i| \le \widetilde{O}((\sqrt{k/t}+ 1/k^{3/2})$, where the first term account for the variance and the second is due to the bias (difference between $\Exp[\hat{h}_i]$ and $h_i$). 
	%$$\Pr\left[|\hat{h}_i -  h_i| \ge \widetilde{\Omega}((\sqrt{k/t}+ 1/(tk^{3/2}))\mid h\right] \le n^{-10}\,.$$
	
	Noting that $k < t < k^2$ and therefore $1/k^{3/2}$ is dominated by $\sqrt{k/t}$, take expectation over $h$, we obtain that $\Pr\left[|\hat{h}_i -  h_i| \ge \widetilde{\Omega}(\sqrt{k/t})\right] \le n^{-10}$. Taking union bound over all $i\in [n]$, we obtain that with high probability, $\|\hat{h} - h\|_{\infty} \le \widetilde{O}(\sqrt{k/t})$.  Recall that $\|h\|$ was assumed to be equal $k$ without loss of generality. Hence we obtain that $\|\hat{h} - h\|_{\infty} \le \widetilde{O}(\sqrt{1/t})\|h\|$ as desired.

\end{proof}

\begin{proof}[Proof of Theorem~\ref{thm:dropout}]
	%Almost the same as Theorem~\ref{thm:single}, 
	Recall that in the generative model already many bits are dropped, 
	% in the generative direction, 
	and nevertheless the feedforward direction maps it back to $h$. Thus dropping some more bits in $x$ doesn't make much difference. Concretely, we note that $x^{\drop}$ has the same distribution as $s_{t/2}(r(\alpha Wh))$. Therefore invoking Theorem~\ref{thm:single} with $t$ being replaced by $t/2$, we obtain that there exists $b'$ such that $\|r(2W^Tx^{\textrm{drop}} + b')- h\|^2 \le \tilO(k/t)\cdot \|h\|$.%-- the nature of our generative model essentially captures the dropout. 
	%\Tnote{The proof requires a slight change of notation}
\end{proof}

\begin{lemma}\label{lem:concentration} Under the same setting as Theorem~\ref{thm:single}, let $\hat{h} = \alpha W^Th$. Then we have
	$$\Pr\left[|\hat{h}_i -  h_i| \ge \widetilde{\Omega}((\sqrt{k/t}+ 1/k^{3/2})\mid h\right] \le n^{-10}$$
\end{lemma}

\begin{proof}[Proof of Lemma~\ref{lem:concentration}]
		Recall that $\hat{h}_i =  \sum_{j\in T} \alpha W_{ji}r(W_{ji}h_j + \eta_j)$. For convenience of notation, we condition on the randomness $h$ and $T$ and only consider the randomness of $W$ implicitly and omit the conditioning notation. Let $Z_j$ be the random variable $\alpha W_{ji}x_j = \alpha W_{ji}r(W_{ji}h_j + \eta_j)$. Therefore $Z_j$ are independent random variables (conditioned on $h$ and $T$). Our plan is to show that $Z_j$ has bounded Orlicz norm and therefore $\hat{h}_i = \sum_j Z_j $ concentrates around its mean. 
 Since $W_{ji}$ and $x_j$ are Gaussian random variables with variance 1 and $\|h\|=  \sqrt{k}$, we have that $\|W_{ji}\|_{\psi_2} = 1$ and $\psinorm{x_j}{2}\le \sqrt{k}$. This implies that $W_{ji}x_j$ is sub-exponential random variables, that is, $\psinorm{W_{ji}x_j}{1}\le O(\sqrt{k})$. That is, $Z_j$ has $\psi_1$ orlicz norm at most $O(\alpha \sqrt{k})$. Moreover by Lemma~\ref{lem:psi-norm-mean-shift}, we have that $\psinorm{Z_j-\E[Z_j]}{1} \le O(\alpha \sqrt{k})$. %Finally we note that $Z_j$'s are independent random variables and therefore 
 Then we are ready to apply bernstein inequalities for sub-exponential random variables (Theorem~\ref{thm:berstein-psi1}) and obtain that for some universal constant $c$, 
	
	\begin{equation*}
		\Pr\left[\left|\sum_{j\in T} Z_j - \Exp\left[\sum_{j\in T} Z_j\right]\right| > c\sqrt{|T|}\alpha \sqrt{k}\log n\right] \le n^{-10}\,.
	\end{equation*}
	
	Equivalently, w obtain that 
	
		\begin{equation*}
		\Pr\left[\left|\hat{h}_i - \Exp\left[\hat{h}_i\right]\right| > c\sqrt{|T|k/t^2}\log n\mid h,T\right] \le n^{-10}\,.
		\end{equation*}
		
		Note that with high probability, $|T| = (1\pm o(1)) t$. Therefore taking expectation over $T$ and taking union bound over $i\in [n]$ we obtain that for some absolute constant $c$,  
	\begin{equation*}
	\Pr\left[\forall i\in [n], \left|\hat{h}_i - \Exp\left[\hat{h}_i\right]\right| > c\sqrt{k/t}\log n\mid h\right] \le n^{-8}\,.
	\end{equation*}
	
	Finally note that as shown in the proof of Theorem~\ref{thm:single}, we have $\Exp[\hat{h}_i\mid h] = h_i \pm 1/k^{3/2}$. Combining with the equation above we get the desired result. 
	%Recall the definition of $Z_j = W_{ji}x_j$, $\hat{h_i} = \sum_{j\in T}Z_j$, and the fact that $\Exp[Z_j] = h_i/2 \pm \tilO(1/\|h\|^3)$, we obtain that with high probability over the randomness of $W$, $$\|\hat{h}_i - h_i\|\le \tilO\left(\frac{\|h\|}{\sqrt{t}}
%	\right)+ \tilO\left(\frac{1}{\|h\|^3}\right)$$
	
%	Note that under our assumption we have that $\frac{1}{\|h\|^3}\le O(\frac{\|h\|}{\sqrt{t}})$ when $\|h\|^4  =\Theta(k^2)\ge t$. Therefore the proof is complete. 

	%Moreover, by Lemma~\ref{lem:exp} gain, the variance of $W_{ji}x_j$ can be simply bounded by $\Exp[W_{ij}^2x_j^2]\le 3h_i^2 + \|h\|^2 - h_i^2 = \|h\|^2 + 2h_i^2$. 
	%	\begin{align*}
	%	\mathbb{V}[W_{ji}x_j]&\le \Exp[W_{ji}^2x_j^2]\le \Exp\left[W_{ji}^2(W_{ji}h_i + \sum_{t\neq i}W_{jt}h_t)^2\right] \\
	%	& = \Exp[W_{ji}^2]
	%	\end{align*}
	
	%The proof simply follows f$rom the fact that 
\end{proof}

\begin{claim}\label{lem:app:var}
	Under the same setting as in the proof of Theorem~\ref{thm:single}. The variance of $\hat{h}_i$ can be bounded by 
	\begin{equation}
	\Var[\hat{h}_i \vert h]  \le O(k/t) \nonumber
	\end{equation}
\end{claim}
\begin{proof}
	Using Cauchy-Schwartz inequality we obtain 
	
	\begin{align*}
	\Var[W_{ji}x_j\vert h]  & \le \Exp[W_{ji}^2x_j]^2\vert h] \\
	& \le \Exp[W_{ji}^4\vert h]^{1/2}\Exp[x_j^4\vert h]^{1/2}\le O(k)
	\end{align*}
	It follows that 
	\begin{align*}
	\Var[\hat{h}_i \vert h] & = \Exp_T\left[\sum_{j\in T}\alpha^2 \Var[W_{ji}x_j\vert h]\right] \le \alpha^2 t\cdot O(k) = O(k/t),  \nonumber
	\end{align*}
	
	%where we used Cauchy-Schwartz inequality, that is,  $\Var[W_{ji}x_j\vert h]  \le \Exp[W_{ji}^2x_j]^2\vert h] \le \Exp[W_{ji}^4\vert h]^{1/2}\Exp[x_j^4\vert h]^{1/2}\le O(k)$ 
	
	We note that a tighter calculation of the variance can be obtained using Lemma~\ref{lem:exp}). 
\end{proof}
\section{Auxiliary lemmas for one layer model }\label{sec:onelayertechnical}

We need a series of lemmas for the one layer model, which will be further used in the proof of two layers and three layers result. %Throughout this section, we use the same setting in Section~\ref{sec:onelayertechnical}. 

As argued earlier, the particular scaling of the distribution of $h$ is not important, and we can assume WLOG that $\|h\|^2 = k$. Throughout this section, we assume that for $h\sim D_h$, 
\begin{eqnarray}
h \in \mathbb{R}_{\ge 0}^n, \quad |h|_0 \le k \, ,\,\|h\| =k\,,\, \textrm{ and }\, |h|_{\infty} \le O(\sqrt{\log N})\quad \textrm{almost surely} \label{eqn:scaledassumption}
\end{eqnarray}
%\Tnote{To check the notation}
\begin{lemma}\label{lem:exp}
	Suppose $w\in \mathcal{N}(0,1)$ and $\xi\sim \mathcal{N}(0,\sigma^2)$ are two independent variables. For $\sigma = \Omega(1)$ and $0\le h\le \log(\sigma)$, we have that $\Exp\left[w\cdot \relu(wh+ \xi)\right] = \frac{h}{2} \pm \tilO(1/\sigma^3)$ and $\Exp[w^2\cdot \relu(wh+ \xi)^2] \le 3h^2+\sigma^2$. 
\end{lemma} 

\begin{proof}
	Let $\phi(x) =\frac{1}{\sigma\sqrt{2\pi}}\exp(-\frac{x^2}{2\sigma^2})$ be the density function for random variable $\xi$. Let $C_1 = \Exp[|\xi|/\sigma] =\frac{2}{\sigma}\int_{0}^{\infty} \phi(x)xdx$. We start by calculating $\Exp\left[\relu(wh+\xi)\mid w\right]$ as follows: 
	
	\begin{align*}
	\Exp\left[w\cdot \relu(wh+\xi)\mid w\right] &= w\int_{-wh}^{\infty}\phi(x)(wh+x)dx \\
	& = w\int_{0}^{\infty}\phi(x)(wh+x)dx + w\int_{-wh}^{0}\phi(x)(wh+x)dx \\
	& = \frac{C_1\sigma w}{2} + \frac{w^2h}{2}  + \underbrace{w\int_{0}^{wh}\phi(y)(wh-y)dy}_{G(w)}\\
	\end{align*}
	Therefore we have that 
	\begin{align*}
	\Exp\left[\Exp\left[w\cdot \relu(wh+\xi)\mid w\right]\right] & %= \int_{-wh}^{\infty}\phi(x)(wh+x)dx \\
	%& = \int_{0}^{\infty}\phi(x)(wh+x)dx + \int_{-wh}^{0}\phi(x)(wh+x)dx \\
	= \frac{h}{2}  + \Exp[G(w)]
	\end{align*}
	Thus it remains to understand $G(w)$ and bound its expectation. We calculate the derivative of $G(w)$: We have that $G(w)$ can be written as 
	
	$$G(w) = w^2h\int_{0}^{wh}\phi(y)dy - w\int_{0}^{wh}\phi(y)ydy$$ 
	and its derivative is 
	\begin{align*}
	G(w)' & = 2wh\int_{0}^{wh}\phi(y)dy + w^2h^2\phi(wh) - \int_{0}^{wh}\phi(y)ydy - w^2h^2\phi(wh) \\
	& = 2wh\int_{0}^{wh}\phi(y)dy  - \int_{0}^{wh}\phi(y)ydy 
	\end{align*}
	The it follows that 
	\begin{align*}
	G(w)'' %& = 2wh\int_{0}^{wh}\phi(y)dy + w^2h^2\phi(wh) - \int_{0}^{wh}\phi(y)ydy - w^2h^2\phi(wh) \\
	& = 2h\int_{0}^{wh}\phi(y)dy + 2wh^2\phi(wh)  - wh^2\phi(wh) \\
	& = 2h\int_{0}^{wh}\phi(y)dy + wh^2\phi(wh)  
	\end{align*}
	Moreover, we can get the third derivative and forth one 
	\begin{align*}
	G(w)''' %& = 2wh\int_{0}^{wh}\phi(y)dy + w^2h^2\phi(wh) - \int_{0}^{wh}\phi(y)ydy - w^2h^2\phi(wh) \\
	%& = 2h\int_{0}^{wh}\phi(y)dy + 2wh^2\phi(wh)  - wh^2\phi(wh) \\
	& = 3h^2\phi(wh)+ wh^3\phi'(wh)  
	\end{align*}
	and 
	\begin{align*}
	G(w)^{(4)} %& = 2wh\int_{0}^{wh}\phi(y)dy + w^2h^2\phi(wh) - \int_{0}^{wh}\phi(y)ydy - w^2h^2\phi(wh) \\
	%& = 2h\int_{0}^{wh}\phi(y)dy + 2wh^2\phi(wh)  - wh^2\phi(wh) \\
	& = 4h^3\phi'(wh)+ wh^4\phi''(wh)
	\end{align*}
	Therefore for $h \le 10\log(\sigma)$ and $w$ with $|w|\le 10\log(\sigma)$, using the fact that $\phi'(wh) = \frac{1}{\sigma\sqrt{2\pi}}\exp(-\frac{w^2h^2}{2\sigma^2})\frac{wh}{\sigma^2}\le \frac{2wh}{\sigma^3\sqrt{2\pi}}$ and similarly, $\phi''(wh) \le \frac{1}{\sigma^3\sqrt{2\pi}}$. It follows that $G(w)^{(4)}\le O(\frac{wh^4}{\sigma^3})$ for any $w\le 10\log(\sigma)$. 
	
	Now we are ready to bound $\Exp[G(w)]$: By Taylor expansion at point 0,  we have that $G(w) = \frac{h^2\phi(0)}{2}w^3 + \frac{G(\zeta)^{(4)}}{24}w^4$ for some $\zeta$ between $0$ and $w$. It follows that for $|w|\le 10\log(\sigma)$, $|G(w) - \frac{h^2\phi(0)}{2}w^3|\le \tilO(1/\sigma^3)$, and therefore we obtain that 
\begin{align*}
	\left|\Exp\left[G(w)\mid |w|\le 10\log \sigma\right] - \Exp\left[\frac{h^2\phi(0)}{2}w^3\mid  |w|\le 10\log \sigma\right]\right| & \le \Exp\left[\left|G(w)- \frac{h^2\phi(0)}{2}w^3\right|\mid |w|\le 10\log \sigma\right]  \\
	& \	\le \tilO(1/\sigma^3)
\end{align*}

	Since $w| |w|\le 10\log \sigma$ is symmetric,  we have $\Exp\left[\frac{h^2\phi(0)}{2}w^3\mid  |w|\le 10\log \sigma\right] =0$, and it follows that $\Exp[G(w)\mid |w|\le 10\log \sigma] \le \tilO(1/\sigma^3)$. Moreover, note that $|G(w)|\le O(w^3h^2)\le O(\log^2\sigma w^3)$, then we have that 
	$$\Exp[G(w)\mid |w|\ge 10\log\sigma] \Pr[|w|\ge 10\log\sigma] \le \sigma^{-4}\int_{|w|\ge 10\log\sigma}O(\log^2\sigma w^3) \exp(-w^2/2)dw\le O(1/\sigma^3)$$
	Therefore altogether we obtain 
	\begin{align*}
	\Exp[G(w)] &=\Exp[G(w)\mid |w|\ge 10\log\sigma ] \Pr[|w|\ge 10\log\sigma]+\Exp[G(w)\mid |w|\le 10\log \sigma]\Pr[ |w|\le 10\log \sigma] \\
	& = \tilO(1/\sigma^3) + O(1/\sigma^3) = \tilO(1/\sigma^3). 
	\end{align*}
	
	Finally we bound the variance of $w\cdot\relu(wh+\xi)$. We have that 
	
	\begin{align*}
	\Exp[w^2\cdot\relu(wh+\xi)^2] & \le \Exp[w^2\cdot(wh+\xi)^2]  = h^2\Exp[w^4] + \Exp[w^2]\Exp[\xi^2] = 3h^2 + \sigma^2\,.\\
	\end{align*}
	as desired. 
	
\end{proof}

\begin{lemma}\label{lem:two-correlation}
	For any $a,b\ge 0$ with $a,b\le 5\log \sigma$ and random variable $u,v\sim \mathcal{N}(0,1)$ and $\xi\in \mathcal{N}(0,\sigma^2)$, we have that
	%$\Exp[uv\cdot\relu(au+bv+\xi)^2] =  ab \pm \tilO(1/\sigma)$, and moreover, 
	$\left|\Exp[uv\cdot\relu(au+bv+\xi)^2]  - \Exp[u\relu(au+bv+\xi)]\Exp[v\relu(au+bv+\xi)]\right|=  \tilO(1)$. 
	
\end{lemma}

\begin{proof}
	Let $t = \sqrt{a^2+b^2}$ in this proof. 
	We first represent $u = \frac{1}{t}(ax+by)$ and $v = \frac{1}{t}(bx-ay)$, where $x = \frac{1}{t}(au+bv)$ and $y = \frac{1}{t}(bu-av)$ are two independent Gaussian random variables drawn from $\mathcal{N}(0,1)$. We replace $u,v$ by the new parameterization, 
	\begin{align*}
	\Exp[uv\cdot\relu(au+bv+\xi)^2] &= \Exp\left[\frac{1}{t^2}(ax+by)(bx-ay)\relu(tx+\xi)^2\right] \\
	&= \Exp\left[\frac{1}{t^2} abx^2\relu(tx+\xi)^2\right] -\Exp\left[\frac{1}{t^2} ab\relu(tx+\xi)^2\right]  \\
	&= \Exp\left[\frac{1}{t^2} ab(x^2-1)\relu(tx+\xi)^2\right] 
	\end{align*}
	
	where we used the fact that $\Exp[y] =0$ and $y$ is independent with $x$ and $\xi$. 
	
	Then we expand the expectation by conditioning on $x$ and taking expectation over $\xi$: 
	
	\begin{align*}
	(x^2-1)\Exp[\relu(tx+\xi)^2\mid x] 
	&  = (x^2-1)\int_{-tx}^{\infty} (tx+y)^2\phi(y)dy\\
	&  = (x^2-1)\int_{-tx}^{0} (tx+y)^2\phi(y)dy+ (x^2-1)\int_{0}^{\infty} (tx+y)^2\phi(y)dy\\
	&  = (x^2-1)\int_{0}^{tx} (tx-y)^2\phi(y)dy+ (x^2-1) (C_0t^2x^2+ 2C_1tx + C_2)\\
	& \triangleq H(x) +  (x^2-1) (C_0t^2x^2+ 2C_1tx + C_2)
	\end{align*}
	
	where $C_0 = \int_0^{\infty}\phi(y)dy = \frac{1}{2}$, and $C_1$, $C_2$ are two other constants (the values of which we don't care) that don't depend on $x$. 
	Therefore by Lemma~\ref{lem:H(z)}, we obtain that $|\Exp[H(x)]|\le \tilO(1/\sigma)$ and therefore taking expectation of the equation above we obtain that
	
	\begin{align*}
	\Exp[(x^2-1)\Exp[\relu(tx+\xi)^2\mid x]] 
	%&  = (x^2-1)\int_{-tx}^{\infty} (tx+y)^2\phi(y)dy\\
	%&  = (x^2-1)\int_{-tx}^{0} (tx+y)^2\phi(y)dy+ (x^2-1)\int_{0}^{\infty} (tx+y)^2\phi(y)dy\\
	%&  = (x^2-1)\int_{0}^{tx} (tx-y)^2\phi(y)dy+ (x^2-1) (C_2t^2x^2+ 2C_1tx + C_0)
	& =\Exp[H(x)] +  \Exp[(x^2-1) (C_2t^2x^2+ 2C_1tx + C_0)]\\
	& = 2C_2t^2 \pm \tilO(1/\sigma) = t^2\pm \tilO(1/\sigma)
	\end{align*}
	
	Therefore we have that $\Exp[uv\cdot\relu(au+bv+\xi)^2] = \Exp\left[\frac{1}{t^2} ab(x^2-1)\relu(tx+\xi)^2\right]  = \tilO(1)$. Using the fact that $\Exp[u\relu(au+bv+\xi)]\le \tilO(1)$ and $ \Exp[v\relu(au+bv+\xi)]\le \tilO(1)$, we obtain that 
	
	$\left|\Exp[uv\cdot\relu(au+bv+\xi)^2]  - \Exp[u\relu(au+bv+\xi)]\Exp[v\relu(au+bv+\xi)]\right|=  \tilO(1)$. 
\end{proof}

\begin{lemma}\label{lem:H(z)}
	Suppose $z\sim\mathcal{N}(0,1)$, and let $\phi(x)=\frac{1}{\sigma\sqrt{2\pi}}\exp(-\frac{x^2}{2\sigma^2})$ be the density function of $\mathcal{N}(0,\sigma^2)$ where $\sigma =\Omega(1)$. For $r \le 10\log \sigma$, define $G(z) = \int_{0}^{rz} (rz-y)^2\phi(y)dy$, and $H(z) = (z^2-1)G(z)$. Then we have that $\Exp[|H(z)|]\le \tilO(1/\sigma)$
\end{lemma}
\begin{proof}
	%Essentially we would like to prove that $H(z)$ is close to 0
	Our technique is to approximate $H(z)$ using taylor expansion at point close to 0, and then argue that since $z$ is very unlikely to be very large, so the contribution of large $z$ is negligible. We calculate the derivate of $H$ at point $z$ as follows: 
	\begin{align*}
	H(z)'& = 2zG(z)+ (z^2-1)G'(z) \\
	H(z)'' & = 2G(z) +4zG'(z)+ (z^2-1)G''(z) \\
	H(z)''' & = 6G'(z) + 6zG''(z) + (z^2-1)G'''(z)\\
	%H(z)'''' & = 12G''(z) +  8zG'''(z) + (z^2-1)G''''(z)
	\end{align*}
	Moreover, we have that the derivatives of $G$
	\begin{align*}
	G'(z) & = \int_{0}^{rz}2r(rz-y)\phi(y)dy \\
	G''(z) &= \int_0^{rz}2r^2\phi(y)dy \\
	G'''(z) & = 2r^3\phi(rz) \\
	%G''''(z) & = 2r^4\phi'(rz)
	\end{align*}
	
	Therefore we have that $G'(0) = G''(0) = 0$ and $H'(0) = H''(0) = 0$. Moreover, when $r \le 10\log \sigma$ and $z \le 10\log \sigma$ and we have that bound that $|H'''(z)|\le \tilO(1/\sigma)$. Therefore, we have that for $z\le 10\log \sigma$, $|H(z)|\le \tilO(1/\sigma)$, and therefore we obtain that 
	$\Exp[H(z)\mid |z|\le 10\log \sigma] \Pr[|z|\le 10\log \sigma]\le  \tilO(1/\sigma)$. Since $|H(z)|\le \tilO(1)\cdot z^4$ for any $z$, we obtain that for $z\ge 10\log\sigma$, the contribution is negligible: 	$\Exp[H(z)\mid |z| > 10\log \sigma] \Pr[|z|>  10\log \sigma]\le \tilO(1/\sigma^2)$. Therefore altogether we obtain that $|\Exp[|H(z)|]|\le \tilO(1/\sigma)$. 
\end{proof}

The following two lemmas are not used for the proof of single layer net, though they will be useful for proving the result for 2-layer net. The following lemma bound the correlation between $\hat{h}_i$ and $\hat{h}_j$. 
\begin{lemma}\label{lem:pairwise-correlation}
	%Under the assumption of Theorem~\ref{thm:single-technical}, 
	Let $\hat{h} = W^Tx$, and $D_h$ satsifies~\eqref{eqn:scaledassumption},   then	for any $i,j$, we have that $|\Exp[\hat{h}_i\hat{h}_j] - \Exp[\hat{h_i}]\Exp[\hat{h_j}]|\le \tilO(1/t)$. 
\end{lemma} 

\begin{proof}
	We expand the definition of $\hat{h}_i$ and $\hat{h}_j$ directly. %The key point is that $\hat{h}_i$ and $\hat{h}_j$ are both linear combination of $x_u$'s,
	\begin{align*}
	\Exp[\hat{h}_i\hat{h}_j] &= \alpha^2 \Exp\left[\left(\sum_{u\in S}W_{ui}x_u\right)\left(\sum_{v\in S}W_{vj}x_v\right)\right]\\	
	&= \alpha^2 \sum_{u\in S}\Exp\left[W_{ui}W_{vi}x_u^2\right]  + \alpha^2\sum_{u\neq v}\Exp\left[W_{ui}W_{vj}x_ux_v\right]\\
	& = \alpha^2 \sum_{u\in S}\Exp\left[W_{ui}W_{vi}x_u^2\right]  + \alpha^2\sum_{u\neq v}\Exp\left[W_{ui}x_u\right]\Exp\left[W_{vj}x_ux_v\right]\\
	& = \alpha^2 \sum_{u\in S}\left(\Exp\left[W_{ui}W_{vi}x_u^2\right] - \Exp[W_{ui}x_u]\Exp[W_{vj}x_v]\right)+ \alpha^2\sum_{u, v}\Exp\left[W_{ui}x_u\right]\Exp\left[W_{vj}x_ux_v\right]\\
	& = \alpha^2 \sum_{u\in S}\left(\Exp\left[W_{ui}W_{vi}x_u^2\right] - \Exp[W_{ui}x_u]\Exp[W_{vj}x_v]\right) + \Exp[\hat{h}_i]\Exp[\hat{h}_j]
	\end{align*}
	where in the third line we use the fact that $W_{ui}x_u$ is independent with $W_{vj}x_v$, and others are basic algebra manipulation. Using Lemma~\ref{lem:pairwise-correlation}, we have that 
	$$\left|\Exp\left[W_{ui}W_{vi}x_u^2\right] - \Exp[W_{ui}x_u]\Exp[W_{vj}x_v]\right|\le \tilO(1)$$
	Therefore we obtain that 
	\begin{align*}
	\left|\Exp[\hat{h}_i\hat{h}_j] - \Exp[\hat{h}_i]\Exp[\hat{h}_j] \right|%&= \alpha^2 \Exp\left[\left(\sum_{u\in S}W_{ui}x_u\right)\left(\sum_{v\in S}W_{vj}x_v\right)\right]\\	
	% &= \alpha^2 \sum_{u\in S}\Exp\left[W_{ui}W_{vi}x_u^2\right]  + \alpha^2\sum_{u\neq v}\Exp\left[W_{ui}W_{vj}x_ux_v\right]\\
	%& = \alpha^2 \sum_{u\in S}\Exp\left[W_{ui}W_{vi}x_u^2\right]  + \alpha^2\sum_{u\neq v}\Exp\left[W_{ui}x_u\right]\Exp\left[W_{vj}x_ux_v\right]\\
	%& = \alpha^2 \sum_{u\in S}\left(\Exp\left[W_{ui}W_{vi}x_u^2\right] - \Exp[W_{ui}x_u]\Exp[W_{vj}x_v]\right)+ \alpha^2\sum_{u, v}\Exp\left[W_{ui}x_u\right]\Exp\left[W_{vj}x_ux_v\right]\\
	& = \alpha^2 \tilO(t) = \tilO(1/t)
	\end{align*}
\end{proof}
The following lemmas shows that the 
\begin{lemma}\label{lem:linearcom}
	Under the single-layer setting of Lemma~\ref{lem:pairwise-correlation}, let $K$ be the support of $h$.  Then for any $k$-dimensional vector $u_K$ such that $|u_K|_{\infty}\le \tilO(1)$, we have that $$\Exp[|u_K^T(\hat{h}_K-\Exp \hat{h}_K)|^2] \le \tilO(k^2/t)$$
	%we have %we have that for any vector 
\end{lemma}
\begin{proof}
	%Finally we consider $u_K^T(\hat{h}_K- \Exp[\hat{h}_K])$: 
	%We bound the expectation of its two norm as follows: 
	We expand our target and obtain, 
	\begin{align*}
	\Exp[|u_K^T(\hat{h}_K-\E \hat{h}_K)|^2]  & = \Exp[|\sum_{i\in K}u_i(\hat{h}_i-\E\hat{h}_i)|^2] \\
	& = \sum_{i\in K}\Exp[u_i^2(h_i-\E\hat{h}_i)^2] + \Exp\left[\sum_{i\neq j} u_iu_j(\hat{h}_i-\E\hat{h}_{i})(\hat{h}_j-\E\hat{h}_{j})\right]\\
	& = \sum_{i\in K}\Exp[(h_i-\E\hat{h}_i)^2]\cdot \tilO(1) + \max_{i\neq j}\{|\Exp[\hat{h}_i\hat{h}_j]-\Exp[\hat{h}_i]\Exp[\hat{h}_j]|\} \cdot \tilO(1) \cdot k^2
	%& \le \beta^2\sum_{i\in K}\Exp[(h_i-\E\hat{h}_i)^2] \cdot \tilO(1)+ \beta^2\max_{i\neq j}\{|\Exp[\hat{h}_i\hat{h}_j\mid \mathcal{E}]-\Exp[\hat{h}_i\mid h]\Exp[\hat{h}_j\mid h]|\} \cdot \tilO(1) \cdot k^2\\
	\end{align*}
	%where we used the fact that conditioned on $\mathcal{E}$, $|u|_{\infty}\le \tilO(1)$.  %and $\mathcal{E}$ is independent 
	By Lemma~\ref{lem:app:var}, we have that $\Exp[(h_i-\E\hat{h}_i)^2]\le \tilO(\|h\|^2/t)$. %By Lemma~\ref{lem:two-correlation}, we obtain that $|\Exp[\hat{h}_i\hat{h}_j\mid h]|\le O(\alpha^2 h_ih_j)$, and by Lemma~\ref{lem:single-reversibility}, we have that $|\Exp[\hat{h}_i\mid h]|\le \alpha h_i +\tilO(\alpha\|h\|/\sqrt{s})\le O(\alpha h_i)$. 
	By Lemma~\ref{lem:pairwise-correlation}, therefore we obtain that $\max_{i\neq j}\{|\Exp[\hat{h}_i\hat{h}_j]-\Exp[\hat{h}_i]\Exp[\hat{h}_j]|\}\le \tilO(1/t) $. Using the fact that $\|h\|^2\le \tilO(k)$,  %and $\beta = 2/k$, we obtain that 
	
	$$	\Exp[|u_K^T(\hat{h}_K-\E\hat{h}_K)|^2] \le \tilO(k^2/t)$$
\end{proof}

\section{Multilayer Reversibility and Dropout Robustness}
%\label{sec:multilayer}
We state the formal version of Theorem~\ref{thm:twolayer} here.  Recall that we assume the distribution of $h^{(\ell)}$ satisfies that 
\begin{equation}
h^{(\ell)} \in \R_{\ge 0}^{n_{\ell}}\,, \,|h^{(\ell)}|_0 \le k_{\ell} \textrm{ and } |h^{(\ell)}|_{\infty} \le O\left(\sqrt{\log N/(k_{\ell})}\right)\|h\|\quad \textrm{a.s.} \label{eqn:multi-hdistribution}
\end{equation}
\begin{theorem}[2-Layer Reversibility and Dropout Robustness]
	For $\ell = 2$, and $k_{2} < k_{1} < k_0 < k_2^2$, there exists constant offset vector $b_0, b_1$ such that with probability .9 over the randomness of the weights $(W_0,W_1)$ from prior~\eqref{eqn:multi-prior}, and $h^{(2)}\sim D_2$ that satisfies~\eqref{eqn:multi-hdistribution}, and observable $x$ generated from the 2-layer generative model~\eqref{eqn:mult-generative}, the hidden representations obtained from running the network feedforwardly  
	$$\tilde{h}^{(1)} = r(W_0^Tx+b_1), \textrm{ and }\tilde{h}^{(2)} = r(W_1^Th^{(1)}+b_2)$$
	are entry-wise close to the original hidden variables 
	\begin{equation}
	\forall i\in [n_2], \quad \Exp\left[|\tilde{h}_i^{(2)} - h_i^{(2)}|^2\right] \le \tilO(k_2/k_1)\cdot \bar{h}^{(2)}\label{eqn:entry-wise-errorh2-app}
	\end{equation}
	\begin{equation}
	\forall i\in [n_1], \quad \Exp\left[|\tilde{h}_i^{(1)} - h_i^{(1)}|^2\right] \le \tilO(k_1/k_0)\cdot \bar{h}^{(1)}\label{eqn:entry-wise-errorh1}	
	\end{equation}
	
	where $\bar{h}^{(2) }= \frac{1}{k_{2}}\sum_{i}h^{(2)}_i$ and $\bar{h}^{(1)}= \frac{1}{k_{1}}\sum_{i}h^{(1)}_i$ are the average of the non-zero entries of $h^{(2)}$ and $h^{(1)}$. 
	
	Moreover,  there exists offset vector $b_1'$ and $b_2'$, such that when hidden representation is calculated feedforwardly with dropping out a random subset $G_0$ of size $n_0/2$ in the observable layer and $G_2$ of size $n_1/2$ in the first hidden layer:
	$$\tilde{h}^{(1)\drop} = r(2W_0^T(x\odot n^0_{\drop}) +b_1'), \textrm{ and }\tilde{h}^{(2)} = r(2W_1^T(h^{(1)}\odot n^1_{\drop})+b_2')$$
	are entry-wise close to the original hidden variables in the same form as equation~\eqref{eqn:entry-wise-errorh2-app} and~\eqref{eqn:entry-wise-errorh1}. (Here $n^0_{\drop}$ and $n^{1}_{\drop}$ are uniform random binary vector of size $n_0$ and $n_1$, which governs which coordinates to be dropped). 
	
\end{theorem}

%In this section we formally prove the Theorem~\ref{thm:twolayer}. Again 
For the ease of math, we use a cleaner notation and setup as in Section~\ref{sec:multilayer}. 
%Now we formally prove the correctness of the two layers situation. 
Suppose there are two hidden layers $g\in \mathbb{R}^p$ and $h\in \mathbb{R}^m$ and one observable layer $x$. We assume the sparsity of top layer is $q$,and we assume in our generative model that $h = s_{k}(\relu(\beta Ug))$ and $x =s_t(\relu(\alpha Wh))$ where $U$ and $W$ are two random matrices with standard normal entries. % where we didn't use the scaling as in~\eqref{eqn:mult-generative}. 

Let $\tilde{h} = \relu( W^Tx + b)$ and $\tilde{g} = \relu(U^Th+ c)$, where %$\alpha, \beta$ are two scaling constant, and 
$b$ and $c$ two offset vectors. Our result in the last section show that $\tilde{h} \approx h$, and in the section we are going to show $\tilde{g} \approx g$. The main difficulty here is that the value of $h$ and $\tilde{h}$ depends on the randomness of $U$ and therefore the additional dependency and correlation introduced makes us hard to write $\tilde{g}$ as a linear combination of independent variables. Here we aim for a weaker result and basically prove that $\Exp[|g_{r}-\tilde{g}_{r}|^2]$ is small for any index $r\in [p]$.

\begin{theorem}\label{thm:twolayer-tech}
	When $q< k< t < q^2$, and any fixed $g\in \mathbb{R}^p$ with $|g|_0 \le q$ and $\|g\|_2 =\Theta(\sqrt{q})$ and $|g|_{\infty}\le \tilO(1)$, for $\alpha = \frac{2}{t}$, $\beta = 2/k$ and some $b$, and $c$, with high probability over the randomness of $W$, $\tilde{g}$ satisfies that for any $r\in [p]$, 
	$$\Exp[\|\tilde{g}_r - g_r\|^2] \le \tilO(q/k)$$  %$$\|h-\tilde{h}\|^2\le \tilO\left(\frac{k}{t}\right)\|h\|^2$$ %for $\hat{h} = \alpha W^Tx$  with the scaling constant $\alpha = \frac{2}{s}$, we have that $|\hat{h}-h|_{\infty} \le \tilO\left(\frac{\|h\|}{\sqrt{s}}
\end{theorem}
\begin{proof}
	We fix an index $r \in [p]$, and consider $g_{r}$ and $\tilde{g}_{r}$. Moreover, we fix the choice of random sampling function $s_{k}(\cdot)$ and $s_t(\cdot)$ to the function $s_k(x) = [x_K,0]$ and $s_t(x) = [x_T,0]$, where $K$ and $T$ be subsets of $[m]$ and $[n]$, with size $k$ and $t$, respectively. Let $\hat{g}_{r} =  U_{r}^T\tilde{h}$ where $U_{r}$ is the $r$-th column of $U$, that is, $\hat{g}_{r}$ is the version before shift and rectifier linear. Further more, let's assume $U_{r}^T = u^T = \left[u_1,\dots, u_m\right]$ and $u_K$ is the its restriction to subset $K$. Therefore, $\hat{g}_{r}$ can be written as
	\begin{equation}
	%\hat{g}_{r} = \beta u^T\tilde{h} = \beta u^Th + \beta u^T(h-\tilde{h})\label{eqn:decompose}
		\hat{g}_{r} = u^T\tilde{h} = u^Th + u^T(h-\tilde{h})\label{eqn:decompose}
	\end{equation}
	By Lemma~\ref{lem:beforerelu} (applying to the layer between $g$ and $h$), we know that with high probability over the randomness of $U$, $|u^T h - g_{r}|\le \tilO\left(\sqrt{q/k}\right)$. We conditioned on the event $\mathcal{E}$ that $|u^T h - g_{r}|\le \tilO\left(\sqrt{q/k}\right)$, and that $\|u\|_{\infty}\le \tilde{O}(1)$ in the rest of the proof (note that event $\mathcal{E}$ happens with high probability). 
	
	Now it suffices to prove that $u^T(h-\tilde{h})$ is small. First of all, we note that with high probability over the randomness of $W$, $\tilde{h}$ matches $h$ on the support, therefore we can write $ u^T(h-\tilde{h})= u_K^T(h_K-\tilde{h}_K)$, which turns out to be essential for bounding the error. We further decompose it into \begin{equation}
	u_K^T(h_K-\tilde{h}_K) =  u_K^T(h_K-\Exp[\hat{h}_K]) + u_K^T(\Exp[\hat{h}_K]-\hat{h}_K) +  u^T(\hat{h}_K-\tilde{h}_K), \label{eqn:split}
	\end{equation}and bound them individually. 
	
	First of all, we note that $h_i$ is typically of magnitude $\beta \sqrt{q}$ where $q$ is the sparsity of $g$,and $\|h\| \approx \beta \sqrt{kq}$. We first scale down $h$ by $\beta \sqrt{q}$ and then $h/(\beta\sqrt{q})$ meets the scaling of equation~\eqref{eqn:scaledassumption}. We are going to apply Lemmas in Section~\ref{sec:onelayertechnical} with $h/(\beta \sqrt{q})$. 
	
	by equation~\ref{eqn:expectation}, we have that $|\Exp[\hat{h}_i]/(\beta \sqrt{q})- h_i/(\beta \sqrt{q})|\le \tilO(1/(q^{3/2})) $, and therefore $|\Exp[\hat{h}_K]- h_K|_1 \le \tilO(\beta \sqrt{q}\cdot 1/(q^{3/2}k) \cdot k) = \widetilde{O}(1/q) $. 
	
	Therefore $u_K^T(h_K-\Exp[\hat{h}_K]) \le  |u|_{\infty}|\Exp[\hat{h}_K]- h_K|_1 \le \widetilde{O}(1/q)  $. Moreover, we note that $\tilde{h}_K - \hat{h}_K = b\mathbf{1}_K$ is a constant vector and therefore $u^T(\tilde{h}_K - \hat{h}_K) = b'$ for a constant $b$ (which depends on $r$ and $u$ implicitly). 
	
	Finally we bound the term $u_K^T(\hat{h}_K- \Exp[\hat{h}_K]) $. We invoke Lemma~\ref{lem:linearcom} (with $h/(\beta\sqrt{q})$) and obtain that 
	
	\begin{equation}
	\|\Exp[|u_K^T(\hat{h}_K/(\beta\sqrt{q})- \Exp[\hat{h}_K]/(\beta\sqrt{q}))|^2 \mid h,\mathcal{E}]\|\le \tilO(k^2/t) \nonumber%= \tilO(q/t)
	\end{equation}
	
	It follows that 
		\begin{equation}
		\|\Exp[|u_K^T(\hat{h}_K- \Exp[\hat{h}_K])|^2 \mid h,\mathcal{E}]\|\le \tilO(q/t) \nonumber%= \tilO(q/t)\non
		\end{equation}
	
	%Therefore altogether we obtain that $\Exp[\|\beta u_K^T(h_K-\Exp[\hat{h}_K])\|^2\mid h,\mathcal{E}]\le \tilO(1/t)$.
	 Note that the small probability event has only negligible contribution to the the expectation, therefore bounding the difference and marginalize over $h$, we obtain that $\Exp[\|u_K^T(h_K-\Exp[\hat{h}_K])\|^2]\le \tilO(q/t)$. Therefore we have bounded the three terms in RHS of~\eqref{eqn:split}. 
	 %Recalling equation~\eqref{eqn:decompose}, we obtain that 
	 $\Exp[|\hat{g}_r- u^T\tilde{h}|^2]\le \tilO(1/q+q/k) = \tilO(q/k)$. Note that $g = r(\hat{g} + b)$, therefore for any $b = \epsilon \mathbf{1}$  with $\epsilon \le \tilO(q/k)$, we obtain that $\Exp\left[\|\hat{g}-g\|_{\infty}^2\right] \le \tilO(q/k)$.
	
	%$$\Exp[\|\beta u_K^T(h_K-\Exp[\hat{h}_K])\|^2] \le \Exp_h[\Exp[\|\beta u_K^T(h_K-\Exp[\hat{h}_K])\|^2\mid h, \mathcal{E}]]+ \Pr[\mathcal{E}] \le \tilO(1/t)
\end{proof}
\subsection{Proof Sketch of Theorem~\ref{thm:threelayer}}
As argued in Section~\ref{sec:multilayer}, the drawback of not getting high probability bound in the 2-layer Theorems is that we lose the denoising property of rectifier linear. Note that since $|\hat{g}_r - g_r|$ is only small in expectation, passing through the rectifier linear we obtain $r(\hat{g} +c)$, and we can't argue that $r(\hat{g}+c)$ matches the support $g$ theoretically. (Though experimentally and intuitively, we believe that choosing $c$ to be proportional to the noise would remove the noise on all the non-support of $g$. ) 

However, bounding the error in a weaker way we could obtain Theorem~\ref{thm:threelayer}. The key idea is that the in a three layer network described as in Section~\ref{sec:multilayer}, the inference error of $h^{(3)}$ come from three sources: the error caused caused by incorrectly recover $h^{(1)}$, $h^{(2)}$ and the error of reversing the third layer $W_2$. The later two sources error reduces to two and one layer situation and therefore we can bound them. For the first source of error, we note that by Theorem~\ref{thm:twolayer-tech}, the 
\section{Toolbox}

\begin{definition}[Orlicz norm $\|\cdot\|_{\psi_{\alpha}}$] For $1 \le \alpha < \infty$, let $\psi_{\alpha}(x) = \exp(x^{\alpha})-1$. For $0 <\alpha < 1$, let $\psi_{\alpha}(x) = x^{\alpha}-1$ for large enough $x\ge x_{\alpha}$, and $\psi_{\alpha}$ is linear in $[0,x_{\alpha}]$. Therefore $\psi_{\alpha}$ is convex. The Orlicz norm $\psi_{\alpha}$ or a random variable $X$ is defined as 
	\begin{equation}
		\|X\|_{\psi_{\alpha}} \triangleq \inf\{c\in (0,\infty) \mid \Exp\left[\psi_{\alpha}(|X|/c) \le 1\right]
	\end{equation}
\end{definition}

%\begin{theorem}[Theorem 6.21 of~\cite{ledoux2013probability}]
%	There exists a constant $K_{\alpha}$ depending on $\alpha$ such that for a sequence of independent mean zero random variables $X_1,\dots, X_p$ in $L_{\psi_{\alpha}}$, if $0 < \alpha \le 1$, 
%	\begin{equation}
%		\left\|\sum_{i} X_i\right\|_{\psi_{\alpha}} \le K_{\alpha}\left(\left\|\sum_{i}X_i\right\|_1 + \left\|\max_i \|X_i\|\right\|_{\psi_{\alpha}}\right)
%	\end{equation}
%	and if $1 < \alpha \le 2$, 
%		\begin{equation}
%		\left\|\sum_{i} X_i\right\|_{\psi_{\alpha}} \le K_{\alpha}\left(\left\|\sum_{i}X_i\right\|_1 + (\sum_i \|X_i\|_{\psi_{\alpha}}^{\beta})^{1/\beta}\right)
%		\end{equation}
%		where $1/\alpha + 1/\beta = 1$. 
%\end{theorem}

%\begin{theorem}[Bernstein' inequality for subexponential random variables~\cite{lecue09}]\label{thm:berstein-psi1}
%	There exists an absolute constant $c >0$ for which the following holds: Let $X_1,\dots,X_n$ be $n$ independent mena zero $\psi_1$ random variables. Then for every $t > 0$, 
%	$$\Pr\left[\left|\frac{1}{n}\sum_{i=1}^n X_i \right| > t \right]\le 2\exp\left(-cn \min\left(\frac{t^2}{\bar{v}},\frac{t}{M}\right)\right)$$
%	where $M = \max_i \|X_i\|_{\psi_1}$ and $\bar{v} = n^{-1}\sum_i \|X_i\|_{\psi_1}^2$. 
%\end{theorem}
\begin{theorem}[Bernstein' inequality for subexponential random variables~\cite{lecue09}]\label{thm:berstein-psi1}
	There exists an absolute constant $c >0$ for which the following holds: Let $X_1,\dots,X_n$ be $n$ independent mena zero $\psi_1$ random variables. Then for every $t > 0$, 
	$$\Pr\left[\left|\sum_{i=1}^n X_i \right| > t \right]\le 2\exp\left(-c\min\left(\frac{t^2}{n\bar{v}},\frac{t}{M}\right)\right)$$
	where $M = \max_i \|X_i\|_{\psi_1}$ and $\bar{v} = n^{-1}\sum_i \|X_i\|_{\psi_1}^2$. 
\end{theorem}
%
%\begin{lemma}\label{lem:psi-norm-product}
%	If (possibly correlated) random variables $X$, $Y$ have $\psi_{\alpha}$ orlicz norm bounded by  $\|X\|_{\psi_{\alpha}} \le a$ and $\|Y\|_{\psi_{\alpha}}\le b$ then $\|XY\|_{\psi_{\alpha/2}}\le ab$% have $\psi_{\alpha}$ orlicz norm bounded 
%\end{lemma}
%
%\begin{proof}
%	We have that 
%%	\begin{align*}
%%		\Exp\left[\psi_{\alpha/2}(|X||Y|/ab)\right]\Exp\left[\psi_{\alpha/2}(|X||Y|/ab)\right]
%%	\end{align*}
%	\begin{align*}
%	\Exp\left[\exp\left(\frac{|XY|}{ab}\right)^{\alpha/2}\right] & \le	\Exp\left[\exp\left(\frac{1}{2}\left(\frac{|X|}{a}\right)^{\alpha} + \frac{1}{2}\left(\frac{|Y|}{b}\right)^{\alpha} \right)\right] \\
%	& \le \Exp\left[\frac{1}{2}\left(\exp\left(\frac{|X|}{a}\right)^{\alpha}  + \exp\left(\frac{|Y|}{b}\right)^{\alpha} \right)\right]\\
%	&\le 2
%	\end{align*}
%	where the first inequality uses Cauchy-Schwarz inequality and the second uses the convexity of $\exp(\cdot)$ and the third uses the fact that $X$ and $Y$ has $\psi_{\alpha}$ norm $a$ and $b$, resp. 
%	
%\end{proof}

\begin{lemma}\label{lem:psi-norm-mean-shift}
		Suppose random variable $X$ has $\psi_{\alpha}$ orlicz norm $a$, then $X-\Exp[X]$ has $\psi_{\alpha}$ orlicz norm at most $2a$. 
\end{lemma}
\begin{proof}
	First of all, since $\psi_{\alpha}$ is convex and increasing on $[0,\infty)$, we have that $\Exp\left[\psi_{\alpha}(|X|/a)\right]\ge \psi_{\alpha}(\Exp[|X|]/a)\ge \psi_{\alpha}(|\Exp[X]|/a)$. Then we have that $$\Exp\left[\psi_{\alpha}(\frac{|X-\Exp X|}{2a})\right]\le \Exp[\psi_{\alpha}(\frac{|X|}{2a}+\frac{|\Exp X|}{2a})] \le \Exp\left[\frac{1}{2}\psi_{\alpha}(|X|/a) + \frac{1}{2}\psi_{\alpha}(|\Exp X|/a)\right]\le \Exp[\psi_{\alpha}(|X|/a)]\le 1$$
	where we used the convexity of $\psi_{\alpha}$ and the fact that $\Exp\left[\psi_{\alpha}(|X|/a)\right]\ge \psi_{\alpha}(|\Exp[X]|/a)$.
\end{proof}

% !TEX root = deepnet_main.tex
\section{Experimental Details and Additional Results}
\label{sec:add_exp}

\noindent\textbf{Verification of the random-like nets hypothesis.} 
 Figure~\ref{fig:verify} plots some statistics for the 
%We first show that our assumption is roughly satisfied on neural networks trained in practice. We %took the 
the second fully connected layer in AlexNet (after 60 thousand training iterations).
%, and plotted some statistics in .%(cite)
Figure~\ref{fig:alexnet_fc7_w} shows the histogram of the edge weights, which is close to that of a Gaussian distribution. Figure~\ref{fig:alexnet_fc7_b} shows the bias in the RELU gates. The bias entries are essentially constant (in accord with Theorem~\ref{thm:single}, mostly within the interval [0.9, 1.1]. 
 Figure~\ref{fig:alexnet_fc7_eigen} shows that the distribution of the singular values of the weight matrix is close to the quartercircular law of random Gaussian matrices.

\begin{figure*}
\centering
\subfigure[Histogram of the entries in the weight matrix]{
	\includegraphics[width=0.3\textwidth]{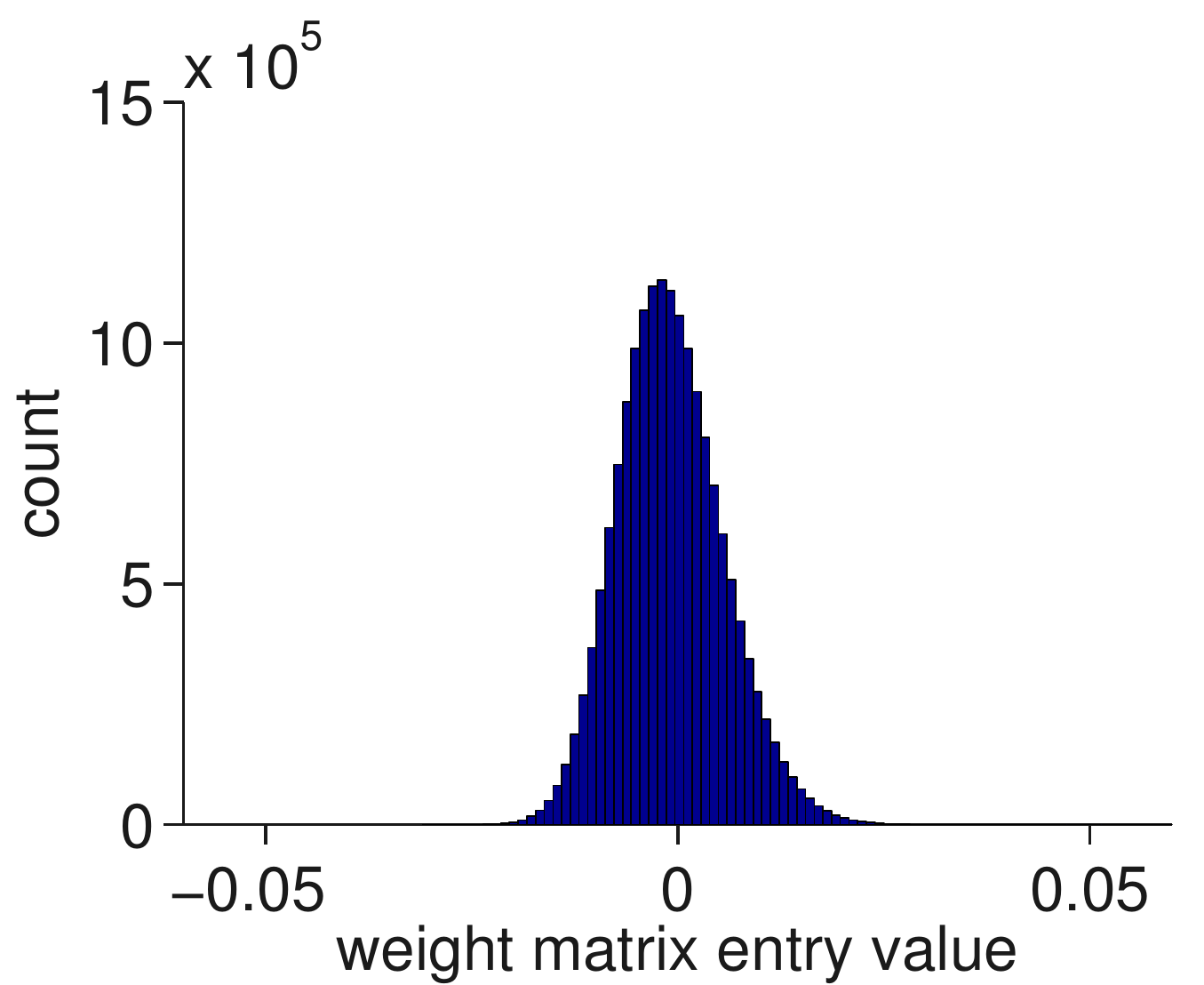} \label{fig:alexnet_fc7_w}
}
\hspace{2mm}
\subfigure[Histogram of the entries in the bias vector ]{
	\includegraphics[width=0.3\textwidth]{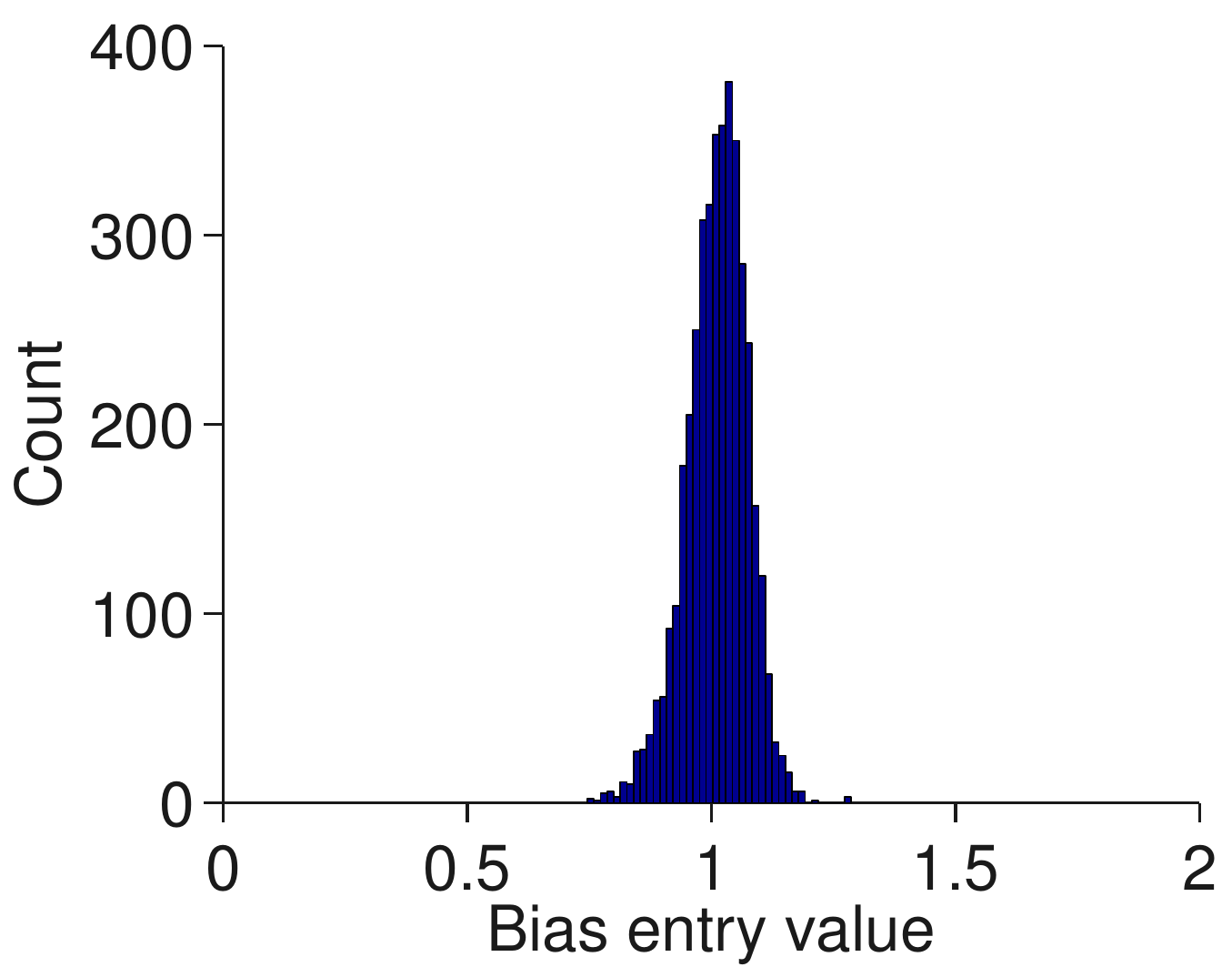} \label{fig:alexnet_fc7_b}
}
\hspace{2mm}
\subfigure[Singular values of the weight matrix]{
	\includegraphics[width=0.3\textwidth]{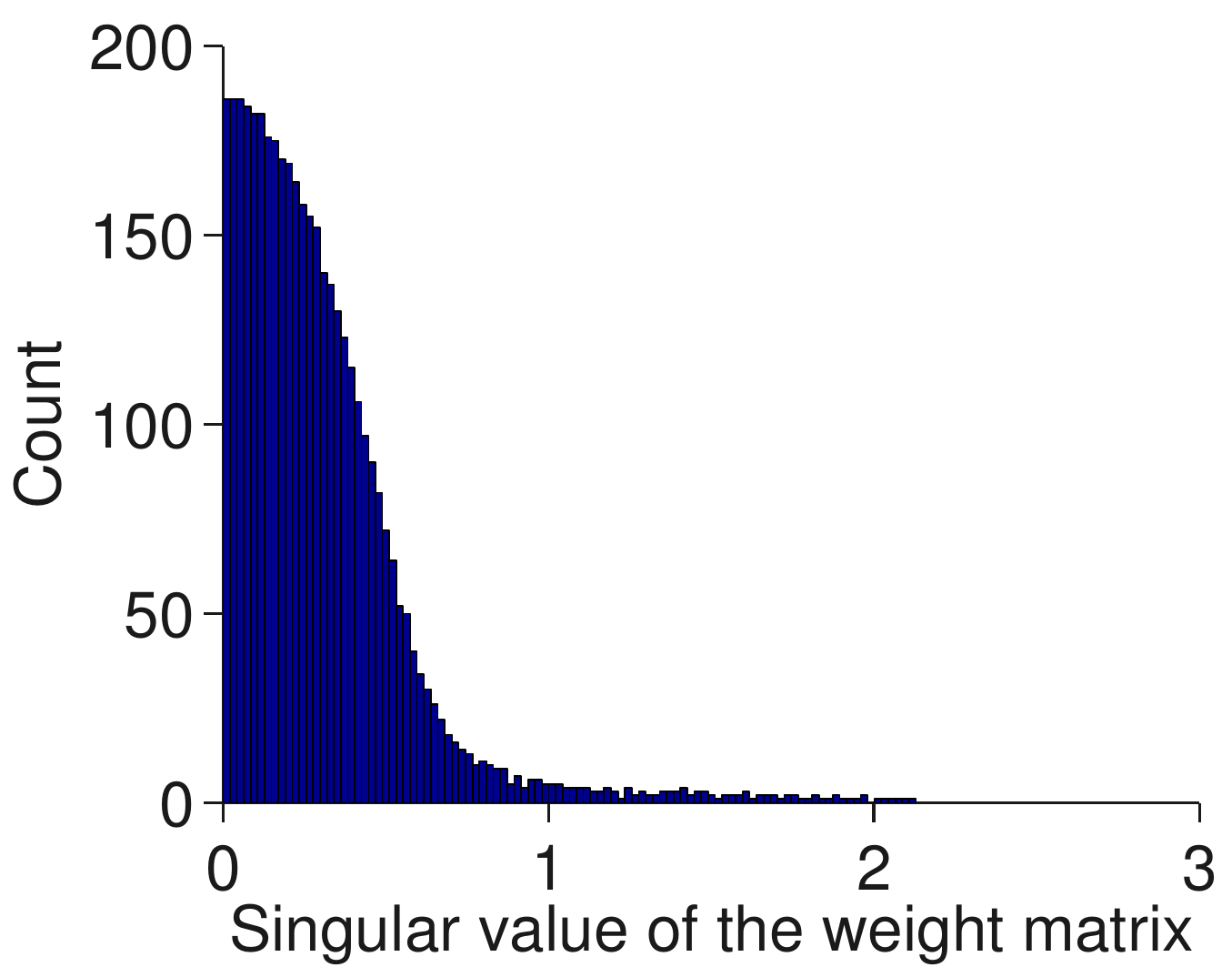} \label{fig:alexnet_fc7_eigen}
}
\caption{Some statistics of the parameters in the second fully connected layer of AlexNet.}
\label{fig:verify}
\end{figure*}

\noindent\textbf{Improved training using synthetic data.} 
The network we trained takes images $x$ as input and computes three layers $h_1, h_2$, and $h_3$ by
$$
	\tilde{h}_1 = r(W_1^\top x + b_1), ~~\tilde{h}_2 = r(W_2^\top h_1 + b_2), ~~\tilde{h}_3 = W_3^\top \tilde{h}_2 + b_3, 
$$
where $h_1$ and $h_2$ have 5000 nodes, and $h_3$ have 10 for CIFAR-10 and MNIST or 100 nodes for CIFAR-100 (corresponding to the number of classes in the dataset).  We generated synthetic data $\tilde{x}$ from $h_2$ by 
$$
	h_1= r(W_2 \tilde{h}_2 ), ~~x' = r(W_1 \tilde{h}_1 ).
$$
Then $x'$ was used for training along with $x$. % See Figure{fig:deinet}.

For training the networks, we used mini-batches of size 100 and the learning rate of $4^\alpha \times 0.001$, where $\alpha$ is an integer between $-2$ and $2$. When training with our regularization, we used a regularization weight $2^\beta$, where $\beta$ is  an integer between $-3$ and $3$. To choose $\alpha$ and $\beta$, 10000 randomly chosen points that are kept out during the initial training as the validation set, and we picked the ones that reach the minimum validation error at the end. The network was then trained over the entire training set. All the networks were trained both with and without dropout. When training with dropout, the dropout ratio is 0.5. 

Figure~\ref{fig:shadow} in the main text shows the test errors with SHADOW combined with backpropagation and dropout.
Here we additionally show the performance with SHADOW combined with backpropagation alone, and also show the test error on the synthetic data generated from the test set. 
The rror drops much faster with SHADOW, and a significant advantage is retained even at the end. The test error on the synthetic data tracks that on the real data, which aligns with the theory.

\begin{figure}
\centering
\subfigure[CIFAR10]{
	\includegraphics[width=0.3\columnwidth]{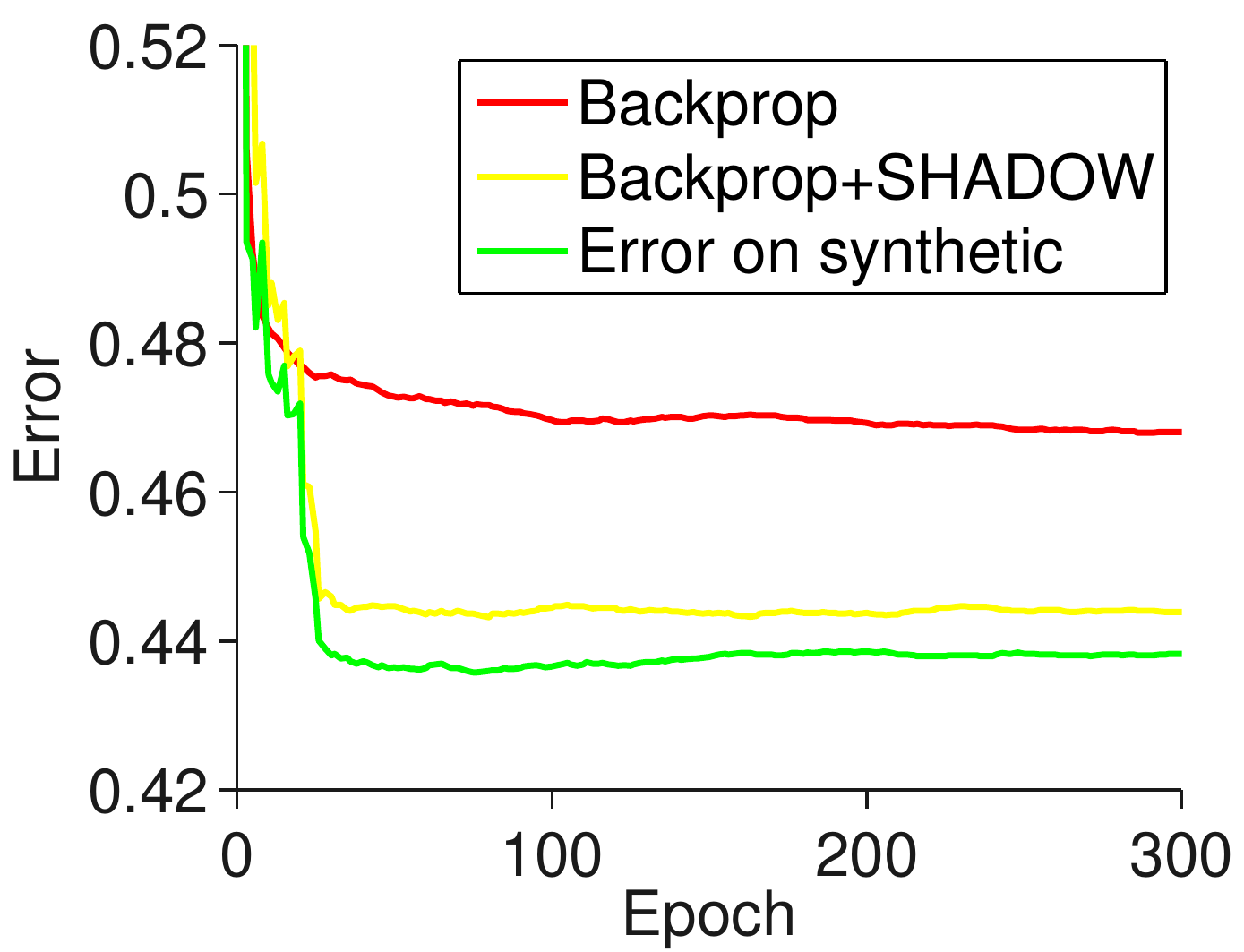}
}
\hspace{2mm}
\subfigure[CIFAR100]{
	\includegraphics[width=0.3\columnwidth]{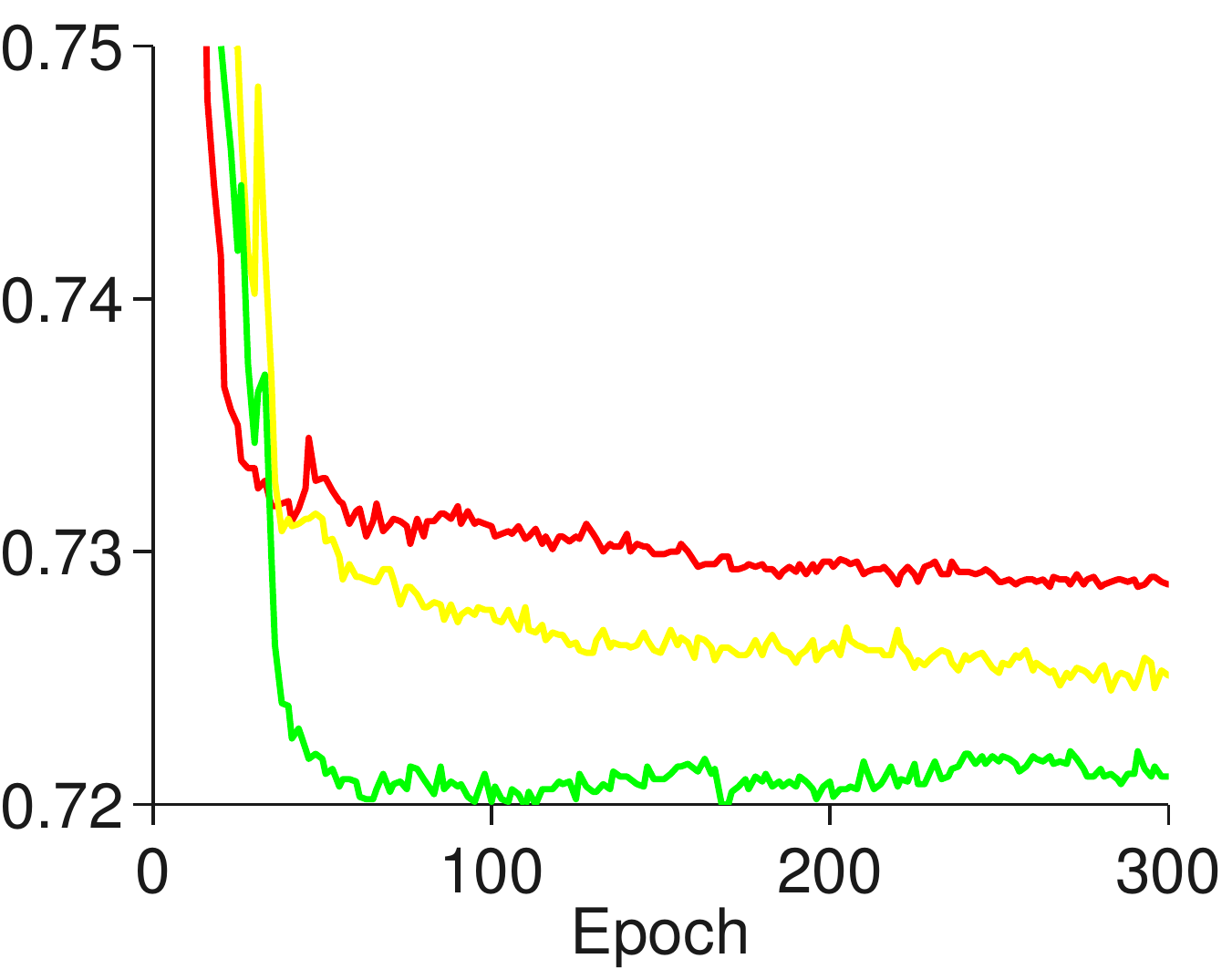}
}
\hspace{2mm}
\subfigure[MNIST]{
	\includegraphics[width=0.3\columnwidth]{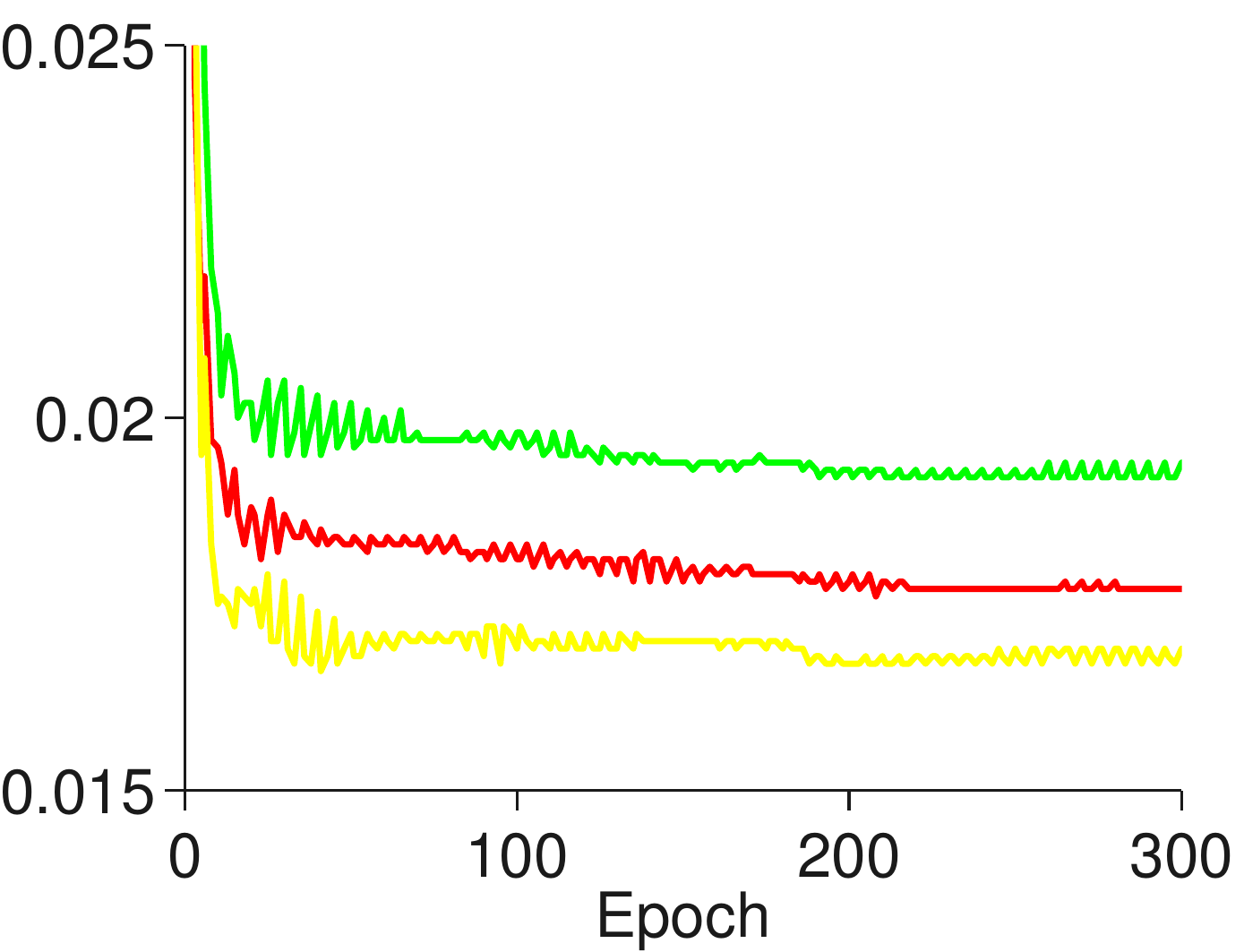}
}
\caption{Testing error of networks trained with and without our regularization for three datasets. Dropout was not used. ``error on synthetic'' is the testing error on the synthetic data generated from the test set.} %TODO: add more description
\label{fig:vanilla_reg20}
\end{figure}

\noindent\textbf{Variants of the SHADOW regularization} 
Instead of using $\tilde{h}_2$ to generate the synthetic data, one can also use other layers. Figure~\ref{fig:reg_diff_layer} compares the results when using $\tilde{h}_3$ and $\tilde{h}_2$.  The error when using $\tilde{h}_3$ is similar or better than that using $\tilde{h}_2$. This indicates that the images generated from $\tilde{h}_3$ can play a similar role as those generated from $\tilde{h}_2$. 
%The error when using $\tilde{h}_3$ has larger variance at the beginning but eventually is lower than that when using $\tilde{h}_2$. This suggests that the regularization allows to escape bad local optimums and leads to a better local optimum.

One can also use sampling when generating synthetic images from the hidden layers. More precisely, one can use 
\begin{equation}
h_1= r(W_2 \tilde{h}_2 ) ~\odot n_{\drop} , ~~x' = r(W_1 \tilde{h}_1) ~\odot n_{\drop}
\end{equation}
where the sampling ratio is 0.5.
This adds more noise to the synthetic data and can also act as a regularization, since the true distribution should be robust to such noise, as suggested by the success of dropout regularization. 
Other prior knowledge about the true distribution can also be incorporated, such as smoothness of the images:
\begin{equation}
h_1= r(W_2 \tilde{h}_2 ) , ~~x' = r(W_1 \tilde{h}_1) , ~~x''=\text{Smooth}(x')
\end{equation}
where \text{Smooth} operator sets each pixel of $x''$ to be the average of the pixels in its $3\times 3$ neighborhood.
The performances of these variants are shown in Figure~\ref{fig:r20_sample_smooth}. With sampling, there is larger variance but the error is similar to that without sampling, in accord with our theory.  With smoothing, the error has larger variance at the beginning but the final error is lower than that without smoothing. 

%\begin{figure}
%\centering
%\begin{minipage}{0.4\columnwidth}
%\includegraphics[width=\columnwidth]{figure/cifar10_reg20_30} \label{fig:reg20_30}
%\caption{Testing error of networks trained with our regularization on CIFAR10 using $h_2$ or $h_3$ to generate synthetic data. Dropout was not used. } %TODO: add more description
%\end{minipage}
%\hspace{4mm}
%\begin{minipage}{0.4\columnwidth}
%\includegraphics[width=\columnwidth]{figure/cifar10_rd20_30} \label{fig:rd20_30}
%\caption{Testing error of networks trained with our regularization on CIFAR10 using $h_2$ or $h_3$ to generate synthetic data. Dropout was used. } %TODO: add more description
%\end{minipage}
%\end{figure}

\begin{figure}
\centering
\subfigure[without dropout]{
\includegraphics[width=0.4\columnwidth]{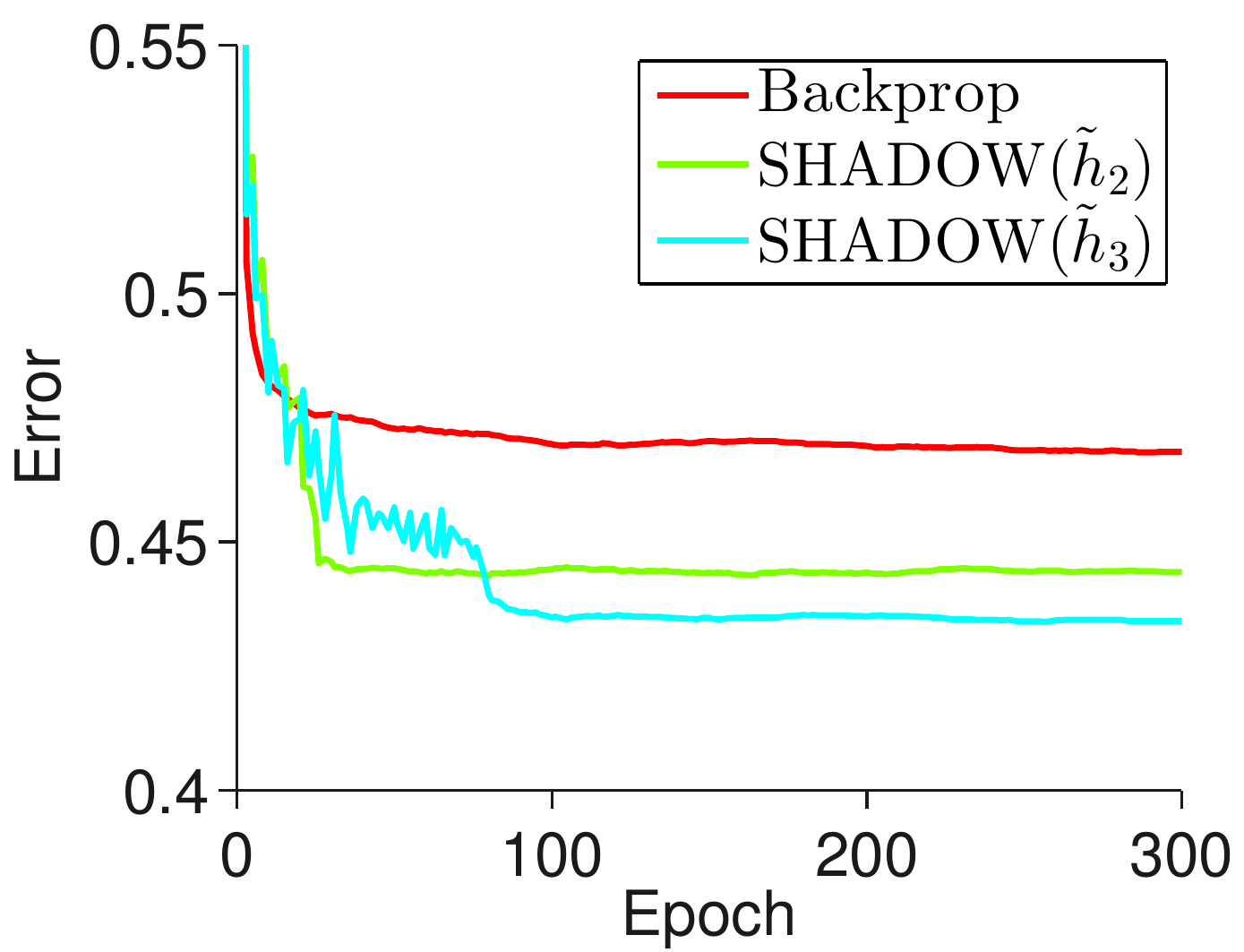} \label{fig:reg20_30}
}
\hspace{2mm}
\subfigure[with dropout]{
\includegraphics[width=0.4\columnwidth]{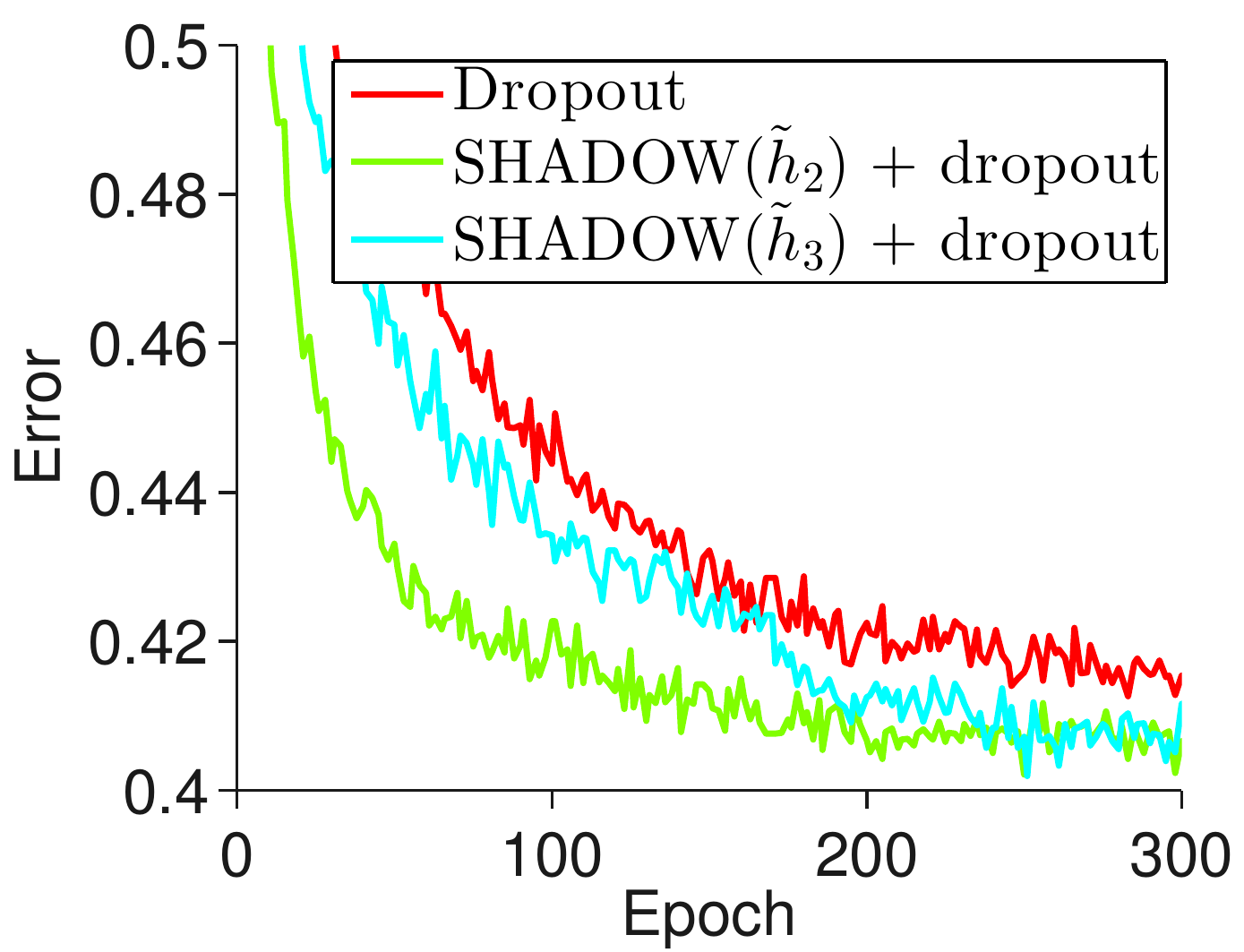} \label{fig:rd20_30}
}
\caption{Testing error of networks trained with our regularization on CIFAR-10 using the hidden layer $\tilde{h}_2$ or $\tilde{h}_3$ to generate synthetic data.} 
\label{fig:reg_diff_layer} 
\end{figure}

\begin{figure}
\centering
\includegraphics[width=0.5\columnwidth]{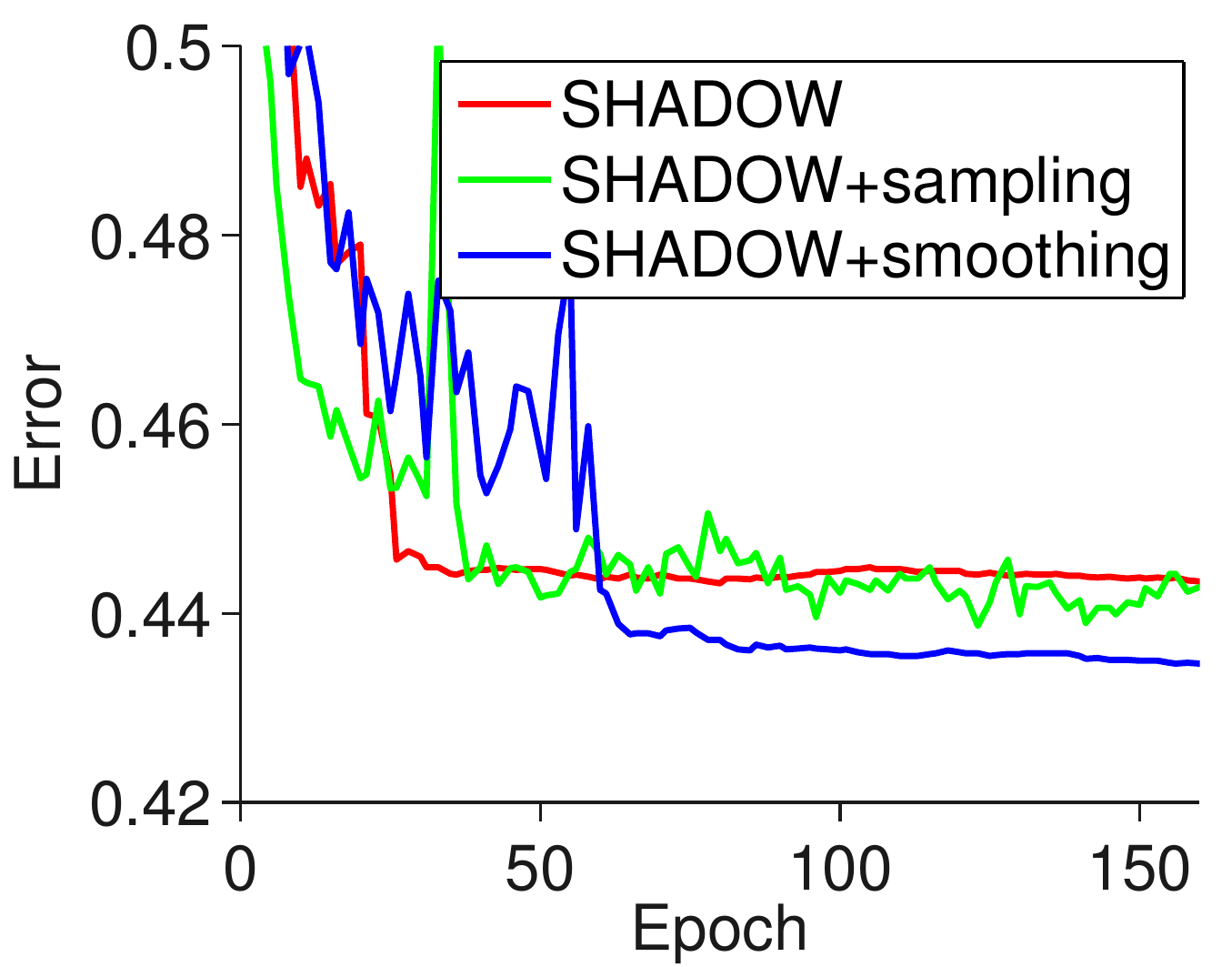} 
\caption{Testing error of networks trained with variants of our regularization on CIFAR-10. SHADOW: uses the hidden layer $h_2$ to generate synthetic data for training. SHADOW+sampling: also uses $\tilde{h}_2$ but randomly sets half of the entries to zeros when generating synthetic data. SHADOW+smoothing: also uses $\tilde{h}_2$ but smooths the synthetic data before using them for training. } \label{fig:r20_sample_smooth}
\end{figure}

\end{document}